\documentclass[journal]{IEEEtran}
\usepackage[utf8]{inputenc}
\usepackage[T1]{fontenc}
\usepackage{amsmath,amssymb,amsthm}
\usepackage{booktabs}
\usepackage{array}
\usepackage{graphicx}
\usepackage{float}
\usepackage{microtype}
\usepackage{hyphenat}
\usepackage[numbers,sort&compress]{natbib}
\usepackage{url}
\usepackage[hidelinks]{hyperref}
\newtheorem{proposition}{Proposition}
\newtheorem{corollary}{Corollary}

\newcommand{\E}{\mathbb{E}}
\newcommand{\M}{\mathcal{M}}

\newcolumntype{L}[1]{>{\raggedright\arraybackslash}p{#1}}

\title{What Does a Routing Oracle Measure Under Stochastic Decoding?
Coupling, Scorer Choice, and Single-Commit Ceilings}

\author{Teng-Ruei~Chen%
\thanks{T.-R. Chen is with (1) the Institute of Bioinformatics and Systems Biology,
National Yang Ming Chiao Tung University, Hsinchu 300, Taiwan; and (2) Krixvon AI,
Taiwan (e-mail: luka@krixvon.com).}}
\begin{document}
\maketitle

\begingroup
\renewcommand{\thefootnote}{}
\footnotetext{Disclosure: Anthropic Claude Code and OpenAI Codex were used as
general-purpose assistive tools for drafting, code implementation and review,
mathematical consistency checks, and language editing.  The research question,
final study-design decisions, and responsibility for every claim, experiment,
citation, and statement in this manuscript remain with the human author(s).}
\endgroup

\begin{abstract}
Large language model routing benchmarks often compare a policy that commits to
one model with a hindsight oracle that credits a query whenever any model's
recorded output is correct.  Under stochastic decoding these are different
decision classes, and the second object is not determined by per-model success
probabilities unless a cross-model coupling is specified.  For a frozen
query--model--decoder--scorer protocol, let $p_{im}$ be model $m$'s success
probability on query $i$.  We distinguish the clairvoyant one-model ceiling
$R_i=\max_m p_{im}$, a named product-coupling union
$U_i^{\perp}=1-\prod_m(1-p_{im})$, and the premium
$\Delta_i^{\perp}=U_i^{\perp}-R_i$.  With only these marginals fixed, the
attainable union probabilities form exactly the interval
$[R_i,\min\{1,\sum_m p_{im}\}]$; hence zero premium is attainable over
couplings compatible with those marginals.  If the protocol instead specifies
independent generation across models, the product formula is its union
probability, not an unidentified quantity.

We apply this measurement audit retrospectively to hash-bound correctness
tensors from 11 open models, 30 archived responses per query--model cell, and
GSM8K, MATH-500, and GPQA-Diamond.  On the designated display channels, the
full-pool empirical product premiums are 0.371, 3.542, and 5.100 percentage
points, respectively.  Their marginally identified premium intervals---whose
zero lower endpoints follow algebraically rather than empirically---are
$[0,0.473]$, $[0,4.787]$, and $[0,7.744]$ points.  Across every retained scorer,
the corresponding ranges are 0.359--0.398, 3.542--3.754, and
5.100--6.632 points.  Eight frozen finite-$k$ paths and exact enumeration of all
$2^{11}-1$ nonempty model subpools expose estimator and pool sensitivity without
turning either spread into a confidence interval.  The designated GSM8K and
GPQA display scorers were developed after limited output inspection and are
labeled accordingly.  An anonymous evidence bundle binds the frozen tensors,
numerical contract, canonical Run-A outputs, and their hashes for independent
number checks.

The main audit measures fixed archived tensors under named scorers and
couplings, not a population effect or an equal-cost routing gain.  A separate
retrospective held-out policy illustration instantiates the decomposition;
it does not establish new-data generalization.  The practical contribution is an
auditable contract for separating one-model ceilings, union assumptions,
scorer choice, and finite-draw sensitivity before interpreting a
router-to-oracle gap.
\end{abstract}


\section{Introduction}

Routing and cascading allocate queries or calls among models in order to trade
quality against cost~\citep{routellm,chen2024frugalgpt}.  Their performance is
often summarized using a benchmark-specific hindsight-oracle variant.  For
example, a binary any-correct oracle may credit a query after observing every
model's answer, whereas some score-based variants take a best observed score.
Routing benchmarks report oracle references under benchmark-specific
conventions; Section~\ref{sec:related-work} explains how our analysis relates
to these references.

That comparison is innocuous when every (query, model) score is a deterministic
property.  It is different under stochastic decoding.  A router selects before
seeing the response on which it is scored; the hindsight oracle selects after
seeing a collection of responses.  Taking a maximum over noisy estimates also
creates the familiar optimizer's curse~\citep{smith2006optimizers}.
Here that maximum estimates the single-commit ceiling.  Under the iid and
cross-model independence assumptions in Section~\ref{sec:estimation}, its
upward bias makes the estimated union-minus-ceiling premium downward biased
in expectation, not upward biased.  This estimation effect is separate from
the difference between decision classes.  The issue is therefore not simply
that an oracle is ``too high.''  The benchmark
has compared two decision rules, with different information and often different
numbers of opportunities to succeed.

This paper asks a narrower and testable question:

\begin{quote}
For a frozen stochastic evaluation protocol, how much does a hindsight union
exceed the best success probability obtainable by committing to one model?
\end{quote}

Figure~\ref{fig:measurement-contract} makes the resulting information boundary
explicit: marginals feed both decision classes, but only the post-response union
requires a cross-model coupling.

\begin{figure*}[t]
\centering
\includegraphics[width=\linewidth]{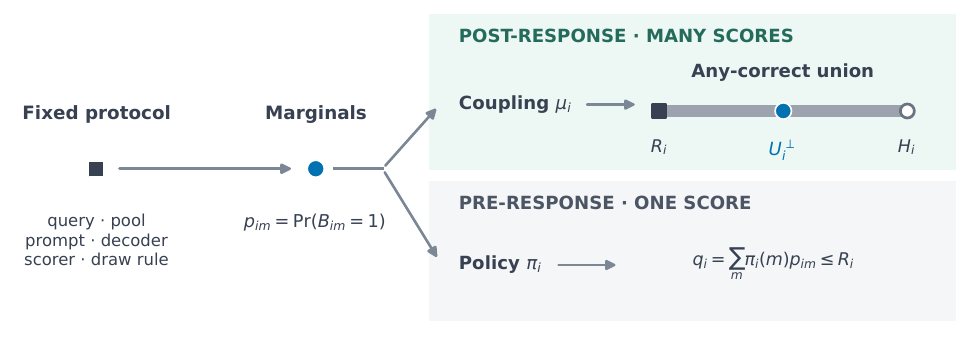}
\caption{Schematic measurement contract used throughout the paper (interval
position not to scale).  The same protocol
induces marginal success probabilities, but a hindsight any-correct union also
requires a coupling and grants credit from a different decision class.  A
pre-response one-model policy is instead capped by the querywise clairvoyant
ceiling $R_i$.}
\label{fig:measurement-contract}
\end{figure*}

The answer has two layers.  First, the mathematical target must name a coupling
across models.  Marginal success probabilities alone do not, in general,
identify the union probability.  We use the product coupling as a transparent reference and keep
same-index pairing only as a descriptive sensitivity.  Second, any ratio to a
router gap must name the router.  The numerator $\bar\Delta^{\perp}$ is independent
of $q$, but the denominator $\bar U^{\perp}-\bar q$ is not.

Our contributions are:

\begin{enumerate}
\item \textbf{A coupling-aware measurement contract.}  We separate a
pre-response one-model ceiling from a hindsight union, give the sharp
Fr\'echet bracket allowed by the marginals, and state exactly which additional
assumption produces the familiar product formula.
Under the stated conditions, we also decompose a named policy's product-reference
gap into distance to this ceiling and a policy-independent product-union premium,
with a matched-protocol reporting recipe.
\item \textbf{A three-task fixed-tensor audit.}  We recompute the same
scorer-conditioned functional on GSM8K, MATH-500, and GPQA-Diamond and report
the absolute product premium together with its coupling-robust interval and
same-index engineering diagnostic.
\item \textbf{Exhaustive sensitivity rather than a favorable slice.}  We report
all retained scorer channels, summarize extrema over every one of the 2,047
nonempty model subpools, report every named fixed-pool check, and retain all
eight frozen finite-$k$ paths.  The complete subpool records are released in
machine-readable form.
\item \textbf{A hash-bound numerical evidence record.}  The submitted package
contains frozen inputs, the numerical contract, canonical Run-A outputs, and
member digests for independent checks of the originally reported numbers;
additional revision outputs are recorded separately
(Appendix~\ref{app:artifact-evidence}).
\end{enumerate}

Appendix~\ref{app:archive-router} adds a held-out policy-procedure illustration
on the historical MATH-500 archive, not a new routing algorithm or a claim of
historically unexposed evaluation data.  We do not claim that clairvoyant
headroom is attainable, interpret a common draw index as a natural coupling,
or aggregate incomparable benchmarks into one headline.

\section{Related Work}\label{sec:related-work}

\paragraph{Routing and routing benchmarks}
Learned routing and cascading seek cost-aware model selection, from FrugalGPT
and Hybrid LLM to RouteLLM~\citep{chen2024frugalgpt,ding2024hybridllm,routellm}.
A unified routing--cascading formulation further makes explicit that choosing
one model and adaptively invoking another are related but operationally
different policies~\citep{dekoninck2025unified}.
RouterBench, RouterEval, and LLMRouterBench provide benchmark-specific model
pools, metrics, and oracle references
~\citep{routerbench,routereval,llmrouterbench}.  Our contribution to interpreting
such references is to separate two questions: what an oracle credits on recorded
responses, and what a stochastic any-correct union is identified to be from
per-model success probabilities.  Proposition~\ref{prop:bracket} gives the
sharp interval $[R_i,H_i]$.  When that interval is non-degenerate, even exact
marginals do not determine a unique union value; eliminating finite-sample
estimation bias would not resolve this identification problem.
Proposition~\ref{prop:decomposition} then relates a named product union to the
ceiling for a policy that commits before seeing its response.  These are
different decision classes even when their numerical values coincide.  This
provides a way to specify an oracle comparison, rather than a new routing
algorithm or a re-evaluation of these benchmarks' published results.

\paragraph{Multi-call inference and repeated sampling}
Self-consistency aggregates multiple sampled reasoning paths, compound AI
systems study how performance scales with additional model calls, and
Mixture-of-Agents combines outputs from several agents
~\citep{wang2023selfconsistency,chen2024morecalls,wang2025moa}.  Repeated
sampling can also expose rare successful responses as inference compute grows
~\citep{brown2024monkeys}.  These systems legitimately use more than one
response and may inspect intermediate outputs.  They are therefore outside our
single-commit policy class, not evidence that a one-call router can attain an
any-correct union.  Our finite-$k$ paths vary the number of archived draws used
to estimate a fixed functional; they are not a best-of-$k$ deployment method.

\paragraph{Stochastic evaluation and selection bias}
LLM scores can vary across repeated executions, even when the nominal protocol
is held fixed~\citep{greedynondet}, and benchmark-level variance can have
multiple sources~\citep{madaan2024variance}.  Maximizing noisy quantities is
a classical source of upward selection bias~\citep{smith2006optimizers}.
Related responses include decoupling selection from evaluation, simultaneous
or winner-focused inference, and general post-selection procedures
~\citep{hasselt2010double,andrews2024winners,berk2013postselection}.  We do not
apply those procedures to manufacture a corrected router score.  Instead, we
separate the policy classes algebraically and expose scorer, coupling, pool,
and finite-draw conditioning in the fixed archive.

\paragraph{Recent audits of multi-model ceilings}
Our contribution relative to recent audits is a joint measurement result:
marginal success probabilities in general identify an interval rather than a
unique hindsight union; every pre-response single-commit policy satisfying
the conditions of Proposition~\ref{prop:decomposition} is capped by $R_i$; and, when the
product-coupling premium is positive, its share of the gap changes with the
named comparison policy $q$ (Corollary~\ref{cor:denominator}).
This complements \citet{fowler2026capability}, who distinguish an oracle
defined by per-query expected performance from perfect-judge post-hoc
selection and develop finite-sample debiasing methods;
\citet{garg2026unsolvability}, who study evaluator, truncation, and formatting
artifacts; and \citet{chen2026cofailure}, who study dependence-aware
co-failure ceilings.  Although $R_i=\max_m p_{im}$, post-hoc union
constructions, generic maximum bias, union bounds, and scorer audits are not
individually new, the results above make their joint measurement consequences
explicit.
The repeated historical tensors then audit the product-coupling plug-in
premium across retained scorer channels and fixed model pools, while treating
finite-$k$ paths as estimator sensitivity rather than a change in the
estimand.

Table~\ref{tab:related-scope} pairs these related scopes with the specific
results used in this audit.

\begin{table*}[t]
\centering
\caption{Scope relative to routing benchmarks and the closest contemporaneous
audits.  Overlap in individual ingredients is acknowledged; the last column
states a concrete result established here rather than a priority claim.}
\label{tab:related-scope}
\small
\begin{tabular}{@{}L{0.18\linewidth}L{0.31\linewidth}L{0.42\linewidth}@{}}
\toprule
Work & Relevant scope & Concrete result established here \\
\midrule
RouterBench, RouterEval, and LLMRouterBench
~\citep{routerbench,routereval,llmrouterbench} & Routing evaluation, including
benchmark-specific model pools, metrics, and oracle references &
Proposition~\ref{prop:bracket} gives the sharp set of stochastic any-correct
unions compatible with fixed marginals, while
Proposition~\ref{prop:decomposition} separates that post-response object from
the pre-response single-commit ceiling. \\
Fowler et al. & Expected-performance oracle, post-hoc selection, and
finite-sample debiasing & Proposition~\ref{prop:bracket} gives the sharp
interval $[R_i,H_i]$.  If non-degenerate, it remains non-degenerate even with
exact marginals, separately from estimation bias. \\
Garg--Sagtani & Evaluator, truncation, and formatting artifacts in reported
ceilings & For a fixed scorer, Proposition~\ref{prop:bracket} characterizes
the attainable union probabilities over couplings with the same marginals.
\\
Chen & Dependence-aware all-model co-failure ceilings &
Proposition~\ref{prop:decomposition} separates
$\bar U^{\perp}-\bar q$ into $\bar R-\bar q$ and
$\bar\Delta^{\perp}$; Corollary~\ref{cor:denominator} makes the corresponding
fraction policy-dependent. \\
\bottomrule
\end{tabular}
\end{table*}

\section{Setup: Three Distinct Objects}\label{sec:setup}

Fix benchmark queries $i=1,\ldots,N$, a model pool
$\M=\{1,\ldots,M\}$, and the complete prompt, model revision, decoding, and
scoring protocol.  A fresh response from model $m$ on query $i$ receives a
binary score
\begin{equation}
B_{im}\sim\operatorname{Bernoulli}(p_{im}).
\end{equation}
The parameter $p_{im}$ is conditional on the complete protocol; it is not an
intrinsic property of the model or query.
Throughout, model pools are nonempty and an empty product equals one.
Statements requiring a positive multi-model premium are therefore vacuous for
$M=1$.

The scorer is part of that protocol.  Formally these quantities are
$p_{im}^{S}$ for a named scorer $S$; we suppress the superscript within one
scorer channel and restore it whenever channels are compared.
For any query-indexed quantity $X_i$, write
$\bar X=N^{-1}\sum_{i=1}^{N}X_i$.

\subsection{A union requires a coupling}

Let $\mu_i$ be a joint distribution for $(B_{i1},\ldots,B_{iM})$ having the
specified Bernoulli marginals.  Its hindsight-union probability is
\begin{equation}
U_i^{\mu}=\Pr_{\mu_i}\!\left(\max_m B_{im}=1\right).
\end{equation}
The marginal probabilities do not, in general, identify $U_i^{\mu}$ without a
joint protocol.  Our primary reference is the specified product coupling,
\begin{equation}
U_i^{\perp}=1-\prod_{m=1}^{M}(1-p_{im}).                 \label{eq:product}
\end{equation}
For this historical archive, it is a reference specified from the marginals,
not an assertion about the actual cross-model sampling mechanism.

\paragraph{A specified joint sampling protocol}
With the complete protocol fixed and a deterministic scorer applied separately
to each response, mutually independent generation randomness across models
gives the product coupling.  In that case,
Eq.~\eqref{eq:product} is the union probability of the specified experiment;
the marginal-identification result is not a reason to discard that value.
A hypothetical non-product mechanism is a coordinated categorical sampler:
at each decoding step, supply the same uniform random draw to every model,
and let each model map it to a token by inverse-CDF sampling under a fixed
ordering of its own token support, using draws independent across steps.
Each model retains its prescribed
marginal token-sampling distribution, while responses across models can become
dependent.  This requires control of the samplers; shared hardware or a common
seed label alone does not specify such a coupling.  Nor does this construction
generally attain the lower endpoint in Proposition~\ref{prop:bracket}.
It illustrates a joint sampling choice, not
how the historical responses analyzed here were generated.

The second object is the clairvoyant single-commit ceiling
\begin{equation}
R_i=\max_m p_{im}.                                      \label{eq:R}
\end{equation}
It assumes knowledge of every $p_{im}$ for this query.  It is an upper bound on
the expected accuracy of a policy that chooses one model before the scoring
response is drawn, not evidence that a feature-based router can attain it.

For compactness, write the sharp marginal upper endpoint as
$H_i=\min\{1,\sum_m p_{im}\}$.

Finally, define the product-coupling oracle premium
\begin{equation}
\Delta_i^{\perp}=U_i^{\perp}-R_i.                       \label{eq:delta}
\end{equation}
This contrast itself crosses decision classes: the union credits success from
multiple responses after observing them, whereas the ceiling concerns committing
to one model.  The same is true of $\bar U^\perp-\bar q$ for a single-commit
policy.  Both are well-defined comparisons under matched protocols, but neither
is, by itself, an equal-response-budget or attainable routing improvement.

\begin{proposition}[Coupling-aware union bracket]\label{prop:bracket}
For every joint coupling $\mu_i$ with Bernoulli marginals $p_{im}$,
\begin{equation}
R_i\le U_i^{\mu}\le H_i.
\end{equation}
Moreover, the set of attainable values of $U_i^\mu$ is exactly this interval.
Consequently $\Delta_i^{\perp}\ge0$.
\end{proposition}

The bracket cannot in general be narrowed from the marginals alone.  For
$p_{i1}=p_{i2}=1/2$, identical success events attain the lower endpoint $1/2$,
whereas disjoint success events attain the upper endpoint $1$.

If $m^\star\in\arg\max_m p_{im}$ and $p_i^\star=p_{im^\star}$, the product
premium also has the exact factorization
\begin{equation}
\Delta_i^{\perp}
=(1-p_i^\star)\left[1-\prod_{m\ne m^\star}(1-p_{im})\right]. \label{eq:factor}
\end{equation}
Under the product reference, the premium is therefore the probability that a
best model fails multiplied by the probability that at least one alternative
succeeds.  This interpretation does not extend automatically to other
couplings.

Figure~\ref{fig:worked-example} illustrates the size of this distinction in a
homogeneous toy case and separates the interval itself from the pool-size trend
under the named product coupling.

\begin{figure*}[t]
\centering
\includegraphics[width=0.98\linewidth]{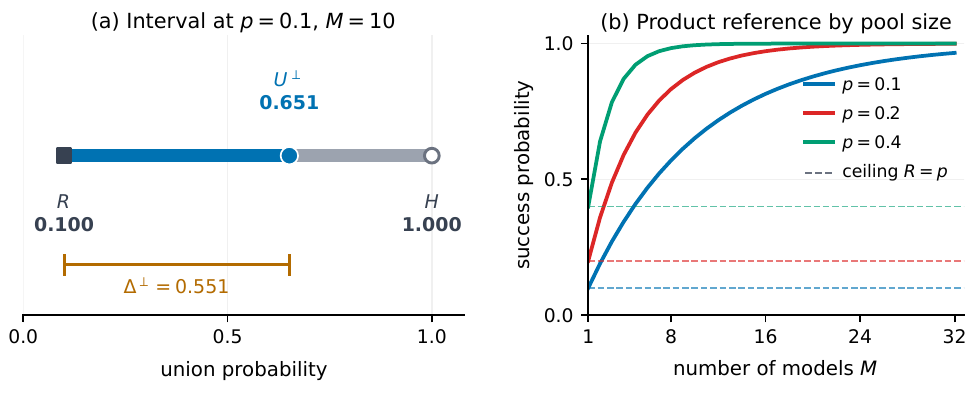}
\caption{Toy example with ten
homogeneous marginals $p_m=0.1$, the single-commit ceiling is $0.1$, whereas
the named product coupling gives union $1-0.9^{10}=0.651$ and premium $0.551$.
The right-hand curve varies model count under this coupling; other couplings
can attain any value in the Fr\'echet interval.}
\label{fig:worked-example}
\end{figure*}

The statement requires no independence.  Independence enters only when the
specific value in Eq.~\eqref{eq:product} is chosen.

\subsection{A named single-commit policy}

Let a policy choose $M_i\sim\pi_i$ using information available before the
scoring response, and assume that choice is conditionally independent of that
fresh response.  Its expected accuracy on query $i$ is
\begin{equation}
q_i=\sum_m\pi_i(m)p_{im}.
\end{equation}

\begin{proposition}[Policy-class ceiling and exact decomposition]
\label{prop:decomposition}
Every such one-model policy satisfies $q_i\le R_i$.  Averaging over the fixed
benchmark,
\begin{equation}
\bar U^{\perp}-\bar q
=\underbrace{(\bar R-\bar q)}_{\text{clairvoyant potential headroom}}
+\underbrace{\bar\Delta^{\perp}}_{\text{product-coupling union premium}}.
\label{eq:decomp}
\end{equation}
Both terms are non-negative, but only the second is independent of the named
comparison policy.
\end{proposition}

This scope excludes cascades, parallel calls, ensembles, and policies allowed
to inspect the response before selecting it.

For $\bar U^{\perp}>\bar q$, define the diagnostic share
\begin{equation}
s(q)=\frac{\bar\Delta^{\perp}}{\bar U^{\perp}-\bar q}.   \label{eq:share}
\end{equation}

\begin{corollary}[The gap fraction is policy-dependent]\label{cor:denominator}
If $\bar\Delta^{\perp}>0$ and
$\bar q_1<\bar q_2\le\bar R$, then $s(q_1)<s(q_2)$.  Therefore a percentage
of ``the routing gap'' is undefined until $q$, its training split, and its
evaluation protocol are named.
\end{corollary}

Within this policy class, $\bar q\le\bar R$.  Hence, when
$\bar\Delta^{\perp}>0$,
$\bar U^{\perp}-\bar q=\bar\Delta^{\perp}+(\bar R-\bar q)
\ge\bar\Delta^{\perp}$ and therefore $0<s(q)\le1$, with equality only when
$\bar q=\bar R$.  Under the stated preconditions $s(q)$ cannot be negative;
$s(q)>1$ requires $\bar q>\bar R$ and therefore a comparison outside the
pre-response single-commit policy ceiling.

\section{Finite Archived Samples}\label{sec:estimation}

The population notation above clarifies the decision classes.  The actual
release, however, begins from a finite archived tensor.  For scorer $S$, let
$B^S_{imj}\in\{0,1\}$ be the stored score of draw $j$.  The mathematical
subscripts are ordered query, model, draw; the serialized tensor axes described
below are ordered model, query, draw.  For a fixed draw-index set $J$ with
$|J|=k$, define
\begin{equation}
\begin{aligned}
\widehat p^S_{im}(J)&=\frac{1}{k}\sum_{j\in J}B^S_{imj},\\
\widehat U^{S,\perp}_{i,A}(J)
&=1-\prod_{m\in A}\{1-\widehat p^S_{im}(J)\},
\end{aligned}
\label{eq:plugin}
\end{equation}
\begin{equation}
\begin{aligned}
\widehat R^S_{i,A}(J)&=\max_{m\in A}\widehat p^S_{im}(J),\\
\widehat\Delta^{S,\perp}_{i,A}(J)
&=\widehat U^{S,\perp}_{i,A}(J)-\widehat R^S_{i,A}(J).
\end{aligned}
\end{equation}
All headline values first apply these nonlinear operations within a query and
then average with equal query weights.  Averaging the cell probabilities before
taking a maximum or product generally gives a different estimand.  For example,
two queries with success vectors $\mathbf p_1=(1,0)$ and
$\mathbf p_2=(0,1)$ have mean querywise ceiling one, whereas the maximum after
averaging the model columns is one half.

Equation~\eqref{eq:plugin} also has an exact fixed-tensor interpretation.  For
each model independently, draw an index uniformly from $J$ and ask whether at
least one selected archived response is correct.  Its conditional probability,
given the tensor, is $\widehat U^{S,\perp}_{i,A}(J)$.  Product is therefore
the exact conditional union probability for this specified empirical
resampling experiment.  Independence is imposed by that rule, not inferred
from the historical generation schedule.

\begin{proposition}[Structural finite-$k$ facts]\label{prop:finitek}
For every binary archived tensor and nonempty pool $A$: (i) at $k=1$,
$\widehat U^{\perp}_{i,A}=\widehat R_{i,A}$ and the plug-in premium is zero;
(ii) when $J$ contains all archived draws, every permutation gives identical
metrics; and (iii) neither the premium nor its components must be monotone
along nested prefixes.
\end{proposition}

The third point matters empirically: a rising path should not be called a
theorem of convergence.  Our finite-$k$ experiment uses eight frozen,
outcome-independent permutations and retains every path.  Their construction
is SHA-256-derived and does not use response outcomes.  The summaries are
descriptive order sensitivity, not independent replications, confidence
limits, or deployment call-budget curves.

Under additional iid within-cell and cross-model independent-stream
assumptions, the product plug-in is unbiased, whereas the plug-in maximum has
nonnegative bias, possibly zero, bounded by $\sqrt{\log(M)/(2k)}$.
Consequently, under those assumptions the plug-in premium has non-positive
bias: it cannot systematically inflate the premium in expectation.  This sign
concerns repeated-sample expectation, not the error direction of this one
archived realization.  At the realized full-pool values $M=11$ and $k=30$, the
worst-case bias-magnitude bound is $0.1999$, or 19.99 percentage points, and is
too loose to be a useful data-specific correction.  We give the conditional result in
Appendix~\ref{app:conditional-bias}.  The archived tensor does not verify its
population assumptions, so we apply no model-based correction.

\section{Retrospective Audit Protocol}\label{sec:methods}

\subsection{Design, queries, and model pool}

This is a historical-pilot-informed retrospective fixed-tensor audit.  The
benchmark queries, model pool, scorer channels, and archived responses are
treated as fixed cases; no query-, model-, or benchmark-population inference is
made.

The tensor contains 30 stored responses for every selected query--model cell
from 11 open-weight instruction models.  The generation logs record
temperature 0.2 and top-$p$ 1.0.  We analyze 500 GSM8K
queries~\citep{gsm8k}, selected uniformly without replacement from the 1,319
test items with NumPy \texttt{default\_rng(42)} and then sorted by source
index; an independent replay exactly reproduces every archived query ID.  We
also analyze the fixed 500-problem MATH test subset released with
\emph{Let's Verify Step by Step}~\citep{mathdataset,lightman2023verify,prm800k}, and all
198 GPQA-Diamond questions~\citep{gpqa}.  The canonical tensor order is model,
query, draw, giving 165,000 binary scores per scorer channel for each math task
and 65,340 for GPQA-Diamond.

The fixed 11-model pool spans Mistral, DeepSeek-Distill-Qwen, three Qwen2.5
sizes, Phi, OLMo, Yi, Granite, Gemma, and Llama variants.  Exact identifiers are
listed in Table~\ref{tab:model-accuracy}.  Historical repository and tokenizer
revisions were not all recorded, so the results characterize archived
responses rather than current repositories.

Table~\ref{tab:protocol-matrix} records the tensor dimensions and every retained
scorer channel used by the audit.

\begin{table*}[t]
\centering
\caption{Frozen audit matrix.  All scorer channels within a benchmark use the same query, model, draw, and raw-response identities; they are alternative deterministic policies, not independent evaluator replications.}
\label{tab:protocol-matrix}
\small
\begin{tabular}{@{}lrrL{1.95in}L{1.55in}@{}}
\toprule
Benchmark & $N$ & Tensor shape & Display scorer (status) & Additional channels \\
\midrule
GSM8K & 500 & $11\!\times\!500\!\times\!30$ & semantic v2 (post-output-developed) & marker-exact v1; legacy lexical \\
MATH-500 & 500 & $11\!\times\!500\!\times\!30$ & Math-Verify (audited candidate) & legacy lexical \\
GPQA-Diamond & 198 & $11\!\times\!198\!\times\!30$ & endorsed-choice v2 (post-output-developed) & final-choice v1; legacy lexical \\
\bottomrule
\end{tabular}
\end{table*}

\subsection{Scorer channels and leakage boundary}

Changing a scorer changes $B^S_{imj}$ and therefore every downstream marginal,
ceiling, union, and premium.  GSM8K and GPQA-Diamond each retain legacy, v1,
and v2 channels; MATH-500 retains a legacy lexical channel and the
\texttt{math500-mathverify} channel backed by Math-Verify
0.9.0~\citep{mathverify}.  Every retained channel is computed on the same
query, model, draw, and raw-response identities within a benchmark, and all are
reported.
These channels are alternative deterministic scoring policies applied to the
same archived responses, not independent evaluator replications.  Their
min--max span is therefore scorer sensitivity, not statistical uncertainty.

The GSM8K and GPQA v2 policies were developed after inspecting limited real
outputs.  They remain labeled post-output-developed, not outcome-blind;
the later human check below does not change that history.  MATH-500's policy uses a bounded boxed-answer
extractor and a frozen Math-Verify configuration, but the release reruns
score-tensor-to-result analysis rather than re-executing raw-text scoring.  The
display channels are presentation labels, not privileged ground truth or a
claim of outcome-blind channel selection.

The retained working records distinguish scorer adoption, premium computation,
and display designation.  On 29 July 2026, the decision to rescore MATH-500
with Math-Verify and official gold preceded the recovered premium computation;
the explicit instruction to use that tensor throughout the MATH tables,
figures, and text followed it.  On 1 August, the table-building code designated
GSM8K and GPQA v2 before the earliest corresponding final-v2 premium
computations recoverable from these records.  Output inspection and scoring
diagnostics had already informed their development.  The 8 August release
contract formalized all three display choices after the premiums were
available.  These recovered events do not establish an outcome-blind initial
display freeze or rule out earlier unrecorded computations.  Every retained channel is reported
in Table~\ref{tab:scorer-sensitivity}.  Scorer disagreement is reported as
evaluator conditioning; it is not automatically annotation error.

\paragraph{Post-review human check}
Three contributors provided two initial judgments per archived response for
60 distinct cells: 30 GSM8K and 30 GPQA-Diamond, with 15 from each frozen-v2
score side per benchmark.  The task was to judge the response's final answer
against the supplied reference, not to independently re-solve the benchmark.
The packets withheld automated scores, model-identifying metadata, and other
contributors' judgments.  Only matching, definite initial judgments were
treated as an analysis consensus.  GSM8K had such a consensus on 26 of 30
cells, all 26 agreeing with frozen v2; GPQA-Diamond had a consensus on 11 of 30,
all 11 agreeing.  The remaining 4 and 19 cells were retained as undetermined,
not excluded from the sample.  Thus the check supplies limited evidence on
the determined subsets, with substantially less coverage on GPQA-Diamond;
it does not validate the full archive or other scorer channels.
Appendix~\ref{app:human-audit} reports the sampling, initial disagreements,
and the decision to forgo planned follow-up after seeing the initial-judgment
summary but before comparison with automated scores.  No scorer or original
audit result was changed.

\subsection{Frozen analysis plan}

For every benchmark and scorer we compute the full-pool empirical product
union, empirical one-model ceiling, absolute premium, Fr\'echet endpoints, and
same-index paired union.  We then enumerate all $2^{11}-1=2{,}047$ nonempty
subpools in canonical bitmask order.  Singleton premiums must be exactly zero.
The release also evaluates fixed provenance-group pools and every
leave-one-declared-group-out pool.  These are sensitivity checks over a fixed
set, not iid draws from a model population.

For finite-$k$ sensitivity, $k\in\{1,2,5,10,15,20,25,30\}$.  Each of eight
SHA-256-derived, outcome-independent permutations supplies nested draw
prefixes, and the same permutation is used across scorer channels within a
benchmark.  The eight permutations are enumerated per benchmark in the released
numerical contract, and all paths are retained.  At $k=30$, every prefix is the full
archive and must reproduce the headline exactly.

Figure~\ref{fig:analysis-pipeline} summarizes this frozen-tensor path from
archived binary scores to the querywise functionals and hash-bound outputs.

\begin{figure*}[t]
\centering
\includegraphics[width=\linewidth]{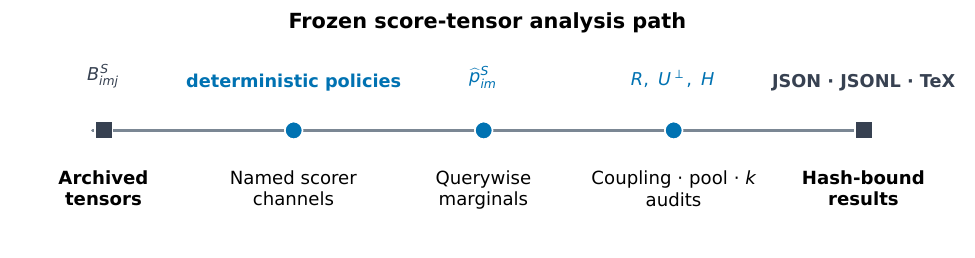}
\caption{Score-tensor-to-paper path.  Nonlinear functionals are computed
querywise before equal-weight averaging within each benchmark; the result
files retain the numerical outputs and their input bindings.}
\label{fig:analysis-pipeline}
\end{figure*}

The original Run-A audit contains no empirical named-policy accuracy or
router-dependent ratio.  The separate revision experiment in
Appendix~\ref{app:archive-router} adds such a policy on a 300/100/100
train/tune/evaluation split of the MATH-500 archive.  Its policy fitting,
tuning, and prediction stages exclude evaluation scores; the author's prior
exposure to the archive is not erased by this split.  This illustration does
not alter the original audit's inputs or results.

\subsection{Released inputs and numerical checks}\label{sec:released-evidence}

The submitted supplement contains the frozen score tensors, their ordered
model/query/draw identities, the numerical contract, canonical Run-A outputs,
and hash records.  These support checks of the score-tensor-to-result
calculations.  The original Run-A checks cover eleven declared scientific invariants,
including complete subpool enumeration, full-draw identity, scorer completeness,
absence of named-policy outputs, and absence of cross-benchmark aggregation.
The absence-of-policy-output invariant describes that original artifact only;
the revision's policy outputs have a separate evidence chain.

The bundle contains an analysis receipt and member digests, but not the
analysis source or referenced release-control files, including
\path{release/source_manifest_v1.json},
\path{release/anonymous_allowlist_v1.json},
\path{configs/reviewed_source_paths_v1.json}, and
\path{release/release_policy_v1.json}.  It therefore does not supply a
reusable implementation for new tensors or the unshared internal execution
procedures.  Appendix~\ref{app:artifact-evidence} distinguishes the submitted
bundle from the descriptive and policy outputs added in this revision.

\section{Results}\label{sec:results}

\paragraph{Scope of the numerical results}\label{sec:result-scope}
All reported quantities describe fixed archived scores under the named scorer,
pool, and coupling, with equal query weights within each benchmark.  Empirical
identification intervals range over couplings compatible with the empirical
marginals; scorer, pool, and prefix ranges describe their respective enumerated
sensitivity sets, not sampling uncertainty.  These quantities are not
population guarantees; the main audit does not measure a router's performance.
The separate policy evaluation is reported in Appendix~\ref{app:archive-router}.
Section~\ref{sec:limitations}
states the shared limits; the text and appendices below retain qualifications specific to each
comparison.

\begin{table*}[t]
\caption{Full-pool display-channel results.  Score levels are percentages;
premiums and differences are percentage points (pp).  Product denotes the
specified coupling reference; Fr\'echet intervals use empirical marginals.}
\label{tab:core-results}
\centering
\small
\begin{tabular}{@{}lrrrrr@{}}
\toprule
Benchmark & $\widehat{\bar R}$ & $\widehat{\bar U}^{\perp}$
& $\widehat{\bar\Delta}^{\perp}$ & Fr\'echet premium
& Paired $-$ product \\
\midrule
GSM8K & 98.980 & 99.351 & \textbf{0.371} & $[0,\,0.473]$ & $-0.011$ \\
MATH-500 & 87.713 & 91.255 & \textbf{3.542} & $[0,\,4.787]$ & $-0.129$ \\
GPQA-Diamond & 82.694 & 87.794 & \textbf{5.100} & $[0,\,7.744]$ & $-0.269$ \\
\bottomrule
\end{tabular}
\end{table*}

Figure~\ref{fig:core-identification} places the product reference and same-index
diagnostic inside the marginally identified interval for each benchmark.

\begin{figure*}[t]
\centering
\includegraphics[width=\linewidth]{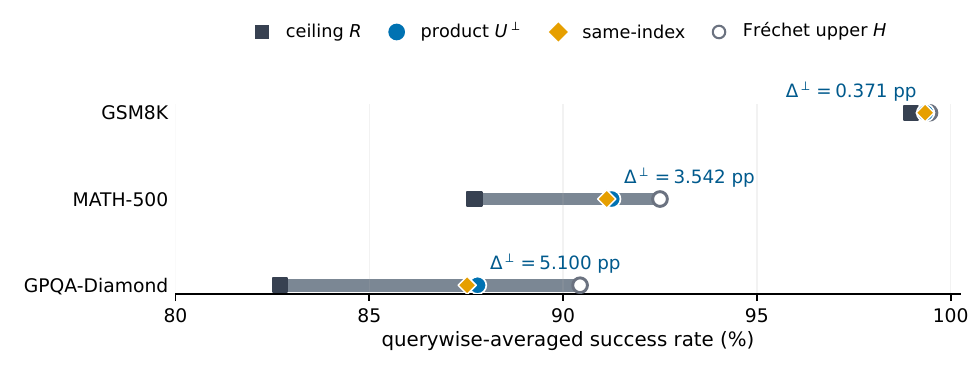}
\caption{Full-pool display-channel identification view (\%).  Gray segments
join the empirical querywise-averaged Fr\'echet endpoints
$[\widehat{\bar R},\widehat{\bar H}]$ (abbreviated in the legend), not confidence
limits.  Blue points mark product references; orange diamonds mark same-index
diagnostics.}
\label{fig:core-identification}
\end{figure*}

\subsection{Where the product-coupled premium occurs}
\label{sec:querywise-premium}

The display-channel full-pool premiums are 0.371 percentage points on GSM8K,
3.542 on MATH-500, and 5.100 on GPQA-Diamond
(Table~\ref{tab:core-results}).  To locate these averages within each
benchmark, we added a querywise analysis in response to review, after seeing
the original results.  It uses the same frozen scores, all 11 models, all 30
archived draws per cell, and every retained scorer channel; the existing
display designations are unchanged.

Fix a benchmark and scorer, and suppress the scorer, full-pool and full-draw
indices in Eq.~\eqref{eq:plugin}.  Equation~\eqref{eq:factor}, applied to the empirical
marginals, gives
\[
\begin{aligned}
\widehat\Delta_i^\perp>0
\quad\Longleftrightarrow\quad&
\max_m\widehat p_{im}<1
\\[-2pt]
&\text{and}\quad
\#\{m:\widehat p_{im}>0\}\ge2.
\end{aligned}
\]
Let $n_+$ count these queries.  For $n_+>0$, define their conditional mean
$\widehat d_+=n_+^{-1}\sum_{i:\widehat\Delta_i^\perp>0}
\widehat\Delta_i^\perp$.  Then, with equal query weights within each benchmark,
\[
\widehat{\bar\Delta}^{\perp}
=\frac{n_+}{N}\widehat d_+.
\]
Thus the average separates exactly into the fraction of queries with a
positive product premium and its mean size on those queries.  If $n_+=0$,
the overall premium is zero and the conditional mean is undefined.

\begin{table*}[ht]
\caption{Querywise product-premium summaries for every retained scorer.
A dagger marks each benchmark's existing display channel.  The conditional
mean uses only positive-premium queries; the overall mean uses all $N$ queries.
Both are in pp, rounded directly from exact values.}
\label{tab:querywise-premium}
\centering
\small
\begin{tabular}{@{}lrrr@{}}
\toprule
Scorer & Positive queries & Conditional mean & Overall mean \\
 & $n_+/N$ & $\widehat d_+$ (pp) & $\widehat{\bar\Delta}^{\perp}$ (pp) \\
\midrule
\multicolumn{4}{@{}l}{GSM8K} \\
\texttt{gsm8k\_v2}$^{\dagger}$ & 13/500 & 14.255 & 0.371 \\
\texttt{gsm8k\_v1} & 17/500 & 10.549 & 0.359 \\
\texttt{gsm8k\_legacy} & 14/500 & 14.207 & 0.398 \\
\midrule
\multicolumn{4}{@{}l}{MATH-500} \\
\texttt{math500\_mathverify}$^{\dagger}$ & 113/500 & 15.672 & 3.542 \\
\texttt{math500\_legacy} & 120/500 & 15.643 & 3.754 \\
\midrule
\multicolumn{4}{@{}l}{GPQA-Diamond} \\
\texttt{gpqa\_v2}$^{\dagger}$ & 73/198 & 13.834 & 5.100 \\
\texttt{gpqa\_v1} & 79/198 & 16.621 & 6.632 \\
\texttt{gpqa\_legacy} & 79/198 & 14.172 & 5.654 \\
\bottomrule
\end{tabular}
\par\smallskip
{\footnotesize Rounding these printed values again can differ from rounding
the exact values directly.}
\end{table*}

On the display channels, positive premiums occur on 13 of 500 GSM8K queries,
113 of 500 MATH-500 queries, and 73 of 198 GPQA-Diamond queries.  Their
conditional means are 14.255, 15.672, and 13.834 pp, respectively
(Table~\ref{tab:querywise-premium}).
For GPQA-Diamond relative to GSM8K on these display channels, the positive-query
fraction is 14.18 times as large, whereas the conditional mean is 0.9705 times
as large, computed from the unrounded summaries.  Thus the larger average
premium in this comparison reflects a larger fraction of positive-premium
queries, not a larger premium on those queries.  This is a factorization of
the fixed-archive averages, not a causal explanation of benchmark differences.

The zero-premium queries fall into three mutually exclusive cases: at least
one model has all 30 archived scores correct; all model counts are zero; or
exactly one model has a positive count, and that count is smaller than 30.  In that order, the
display-channel counts are $(485,0,2)$ for GSM8K, $(372,10,5)$ for MATH-500,
and $(117,5,3)$ for GPQA-Diamond.

These cases also characterize collapse of the empirical Fr\'echet bracket.
Write $\widehat H_i=\min\{1,\sum_m\widehat p_{im}\}$.  If
$\widehat R_i=1$, then $\widehat U_i^\perp=\widehat H_i=1$.  Otherwise,
Eq.~\eqref{eq:factor} makes a zero product premium equivalent to all marginals
other than a chosen maximizer being zero, so $\widehat H_i=\widehat R_i$.
Conversely, the bracket in Proposition~\ref{prop:bracket} gives zero premium
whenever $\widehat H_i=\widehat R_i$.  Hence
\[
\widehat\Delta_i^\perp=0
\quad\Longleftrightarrow\quad
\widehat H_i=\widehat R_i.
\]
The interval constructed from the empirical marginals therefore collapses on
every zero-premium query, so every coupling with these empirical marginals has
the same union value.  This conclusion concerns the frozen scores: archived
counts of 30/30 or 0/30 do not establish population success probabilities of
one or zero, or collapse of the corresponding population bracket.

For GSM8K, the mean querywise plug-in ceiling is
$\widehat{\bar R}=98.980\%$.  The largest in-sample single-model mean
on those same scores is $96.320\%$.  These are two aggregation orders, not
measured performance of a learned router.  Appendix~\ref{app:descriptive-anchor}
reports their difference for each benchmark and scorer.

\subsection{Empirical coupling bounds and the product reference}

For the empirical marginals, the sharp average-premium intervals over compatible
couplings are $[0,0.473]$, $[0,4.787]$, and $[0,7.744]$ percentage points for
GSM8K, MATH-500, and GPQA-Diamond.  The zero lower endpoint follows from the
nested-event construction in the proof of Proposition~\ref{prop:bracket}; the data determine the upper
endpoints.  The positive product values select one specified reference within
these intervals.

If independent sampling is part of the complete protocol, product is the correct
joint reference for that experiment.  The intervals answer a different question:
what the marginals alone identify when a joint protocol is not specified.
The coordinated sampler in Section~\ref{sec:setup} illustrates how marginal
sampling can be preserved while changing cross-model dependence; it is not an
account of this archive.  The empirical bounds do not show that an actual
deployment has zero premium or that the product premium is an attainable
one-call gain.

The archived same-index union lies 0.011, 0.129, and 0.269 points below the
product value, respectively.  These average union values are close to product
in the observed engineering pairing, rather than near the zero-premium
endpoint.  This is positive descriptive evidence about that pairing, not
evidence that the unknown deployment coupling has zero premium.  Nor does
average proximity test independence or establish querywise agreement:
same-index pairing is one empirical coupling, and its difference from product
can have either sign (Appendix~\ref{app:same-index}).

\begin{table*}[t]
\caption{Full-pool product premiums (pp) for every retained scorer on the same
archived responses.  Display designations are presentation labels, not a claim
of outcome-blind selection or independent scorer validation.}
\label{tab:scorer-sensitivity}
\centering
\small
\begin{tabular}{@{}lllr@{}}
\toprule
Benchmark & Channel & Role & Premium \\
\midrule
GSM8K & \texttt{gsm8k\_v2} & audited candidate, post-output developed & 0.371 \\
       & \texttt{gsm8k\_v1} & sensitivity & 0.359 \\
       & \texttt{gsm8k\_legacy} & sensitivity & 0.398 \\
\addlinespace
MATH-500 & \texttt{math500\_mathverify} & audited candidate & 3.542 \\
         & \texttt{math500\_legacy} & sensitivity & 3.754 \\
\addlinespace
GPQA-Diamond & \texttt{gpqa\_v2} & audited candidate, post-output developed & 5.100 \\
             & \texttt{gpqa\_v1} & sensitivity & 6.632 \\
             & \texttt{gpqa\_legacy} & sensitivity & 5.654 \\
\bottomrule
\end{tabular}
\end{table*}

\subsection{The qualitative ordering survives scorer choice}

Across retained channels, premiums range from 0.359 to 0.398 points for GSM8K,
3.542 to 3.754 for MATH-500, and 5.100 to 6.632 for GPQA-Diamond
(Table~\ref{tab:scorer-sensitivity}).  Thus every retained scorer preserves the
qualitative ordering of these three fixed case studies, but GPQA-Diamond has a
substantially wider scorer range.  We do not average channels or benchmarks,
and agreement in ordering does not validate any scorer.

Figure~\ref{fig:scorer-premiums} displays every retained deterministic scorer
channel rather than selecting only the designated display values.

\begin{figure*}[t]
\centering
\includegraphics[width=\linewidth]{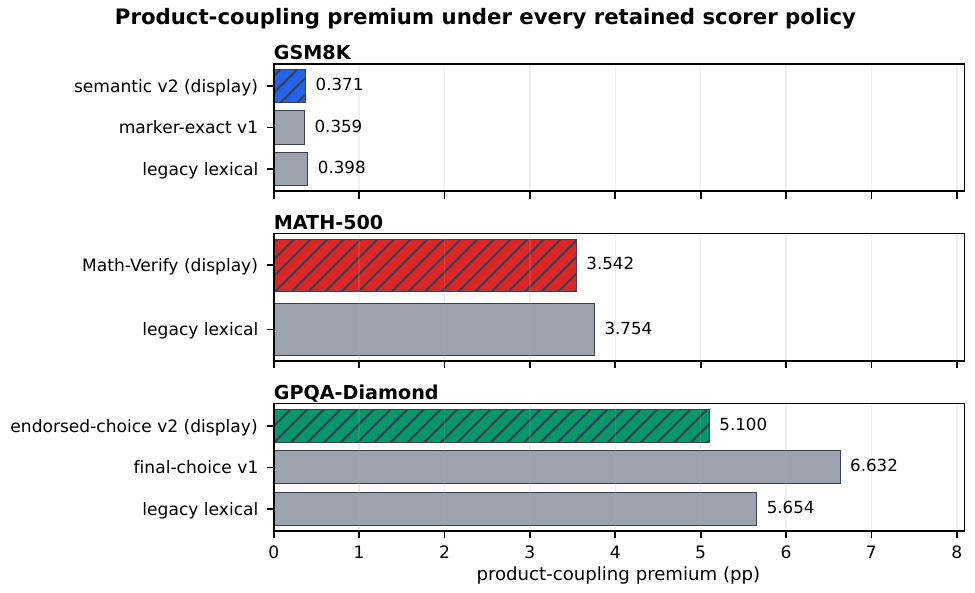}
\caption{Full-pool product premiums (pp) under all retained scorers.
Hatched colored bars mark display channels; gray bars mark other channels on
the same responses.  Their spread describes scorer sensitivity.}
\label{fig:scorer-premiums}
\end{figure*}

\begin{table*}[t]
\caption{Display-channel product premiums (pp) for named fixed pools.
The leave-one-group-out row is the minimum--maximum across nine declared
group removals, describing a fixed sensitivity set.}
\label{tab:fixed-pool-sensitivity}
\centering
\small
\begin{tabular}{@{}lrrr@{}}
\toprule
Pool analysis & GSM8K & MATH-500 & GPQA-Diamond \\
\midrule
Full 11-model pool & 0.371 & 3.542 & 5.100 \\
One per declared group & 0.378 & 3.837 & 6.526 \\
Qwen2.5 size sweep & 0.390 & 2.508 & 0.550 \\
Leave-one-group-out range & 0.326--0.567 & 3.311--3.648 & 4.408--7.729 \\
\bottomrule
\end{tabular}

\vspace{2pt}
{\footnotesize Leave-out minima/maxima are Mistral/Phi for GSM8K,
Granite/Phi for MATH-500, and Granite/Qwen2.5 for GPQA-Diamond.}
\end{table*}

\subsection{Pool composition is a material conditioning choice}

Across the 12 named pool analyses on each display channel---the
full pool, two named fixed subsets, and nine leave-one-declared-group-out
pools---the absolute product premium ranges from 0.326 to 0.567 percentage
points on GSM8K, from 2.508 to 3.837 points on MATH-500, and from 0.550 to
7.729 points on GPQA-Diamond (Table~\ref{tab:fixed-pool-sensitivity}).  The
corresponding full-pool values are 0.371, 3.542, and 5.100 points.  These
frozen fixed-tensor contrasts are descriptive sensitivity analyses,
not optimized deployment policies or draws from a model population.

Figure~\ref{fig:fixed-pools} compares the three named fixed pools and uses both
color and hatching to distinguish them.

\begin{figure*}[t]
\centering
\includegraphics[width=\linewidth]{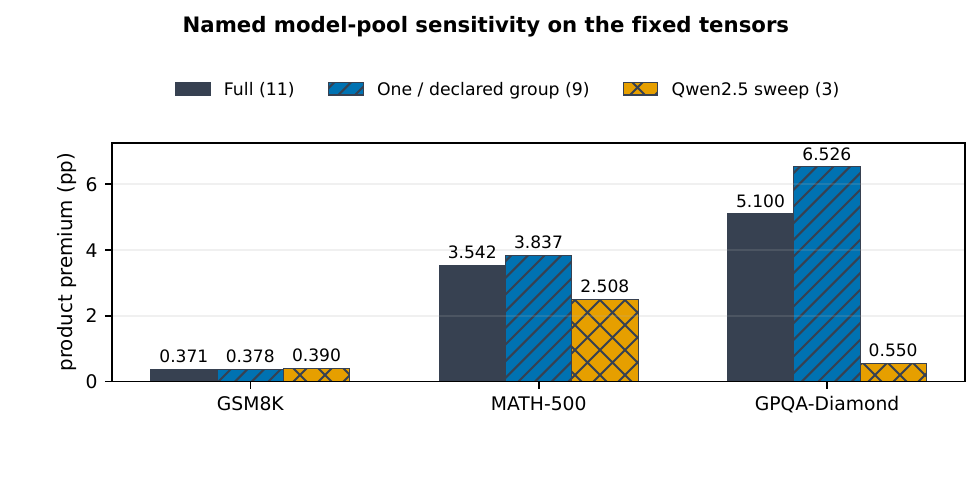}
\caption{Product premiums (pp) for three named fixed pools.  The Qwen2.5
sweep changes both composition and size (three versus eleven models), so the
contrast does not isolate either effect.  The supplement records fixed-pool
results for every retained scorer.}
\label{fig:fixed-pools}
\end{figure*}

\begin{table*}[t]
\caption{Mean display-channel premium across eight fixed nested draw permutations (pp).
Intermediate means diagnose prefix sensitivity.  Here $k$ changes the estimator,
not the estimand: the $k=1$ zero is structural, and agreement at $k=30$ follows
from using the same complete archive, not convergence.}
\label{tab:finite-k}
\centering
\small
\begin{tabular}{@{}lrrrrr@{}}
\toprule
Benchmark & $k=1$ & $k=5$ & $k=10$ & $k=20$ & $k=30$ \\
\midrule
GSM8K & 0.000 & 0.236 & 0.317 & 0.349 & 0.371 \\
MATH-500 & 0.000 & 2.372 & 3.024 & 3.398 & 3.542 \\
GPQA-Diamond & 0.000 & 3.559 & 4.398 & 4.757 & 5.100 \\
\bottomrule
\end{tabular}
\end{table*}

\subsection{Finite-draw paths diagnose plug-in estimator sensitivity}

Here $k=|J|$ controls how many archived draws enter
$\widehat p^S_{im}(J)$ and its derived plug-in quantities.  It changes the
estimator, not the population estimand in Section~\ref{sec:setup}.

Tables~\ref{tab:finite-k} and~\ref{tab:finite-k-full} give two descriptive
checks.  Comparing intermediate means and ranges with the full-archive
endpoint shows sensitivity to retaining only a prefix of the draws.  At a
fixed $k$, the min--max range shows sensitivity to ordering across the eight
specified paths.  The full-archive endpoint is a reference for these
comparisons, not the known population value.

Proposition~\ref{prop:finitek} explains the endpoints: binary empirical
marginals force a zero premium at $k=1$, and all eight paths contain the same
complete archive at $k=30$.  The first is not evidence of a zero population
premium; the second is agreement by design, not evidence of convergence.
Individual paths need not be monotone.  Neither intermediate values nor their
spread provide independent replications, confidence intervals, a deployment
call-budget curve, or a stopping rule for collecting responses
(Section~\ref{sec:limitations}).  Figure~\ref{fig:finite-k} displays the
structural endpoints and intermediate order sensitivity together.

\begin{figure*}[t]
\centering
\includegraphics[width=\linewidth]{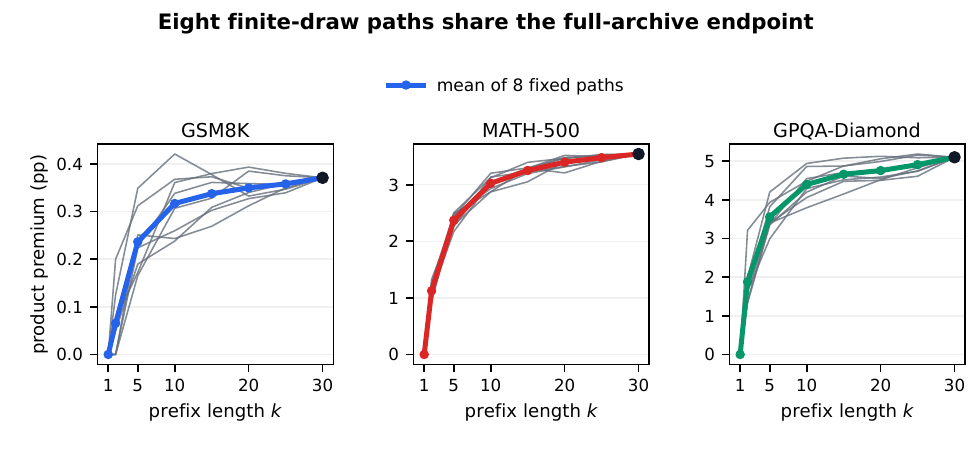}
\caption{Product premiums (pp) along eight fixed nested-prefix paths, with
their mean in bold.  The $k=1$ zero and shared $k=30$ value are structural;
intermediate dispersion diagnoses order sensitivity, not convergence or
uncertainty coverage.}
\label{fig:finite-k}
\end{figure*}

\subsection{The full pool need not maximize the premium}

Adding models makes the product union monotone, but not the premium.  The
largest observed premiums occur at pool sizes 3, 5, and 6
(Table~\ref{tab:subpool}).  As maxima over 2,047 correlated fixed-tensor values,
they are descriptive and potentially selection-optimistic, not deployment
policies or model-population estimates.

\begin{table*}[t]
\caption{Display-channel product premiums over all $2^{11}-1=2{,}047$
nonempty subpools.  Maxima are descriptive, selection-optimistic extrema of
this enumerated set.}
\label{tab:subpool}
\centering
\small
\begin{tabular}{@{}lrrr@{}}
\toprule
Benchmark & Full 11-model premium & Largest subpool premium & Pool size at largest \\
\midrule
GSM8K & 0.371 pp & 2.682 pp & 3 \\
MATH-500 & 3.542 pp & 6.859 pp & 5 \\
GPQA-Diamond & 5.100 pp & 9.627 pp & 6 \\
\bottomrule
\end{tabular}
\end{table*}

Figure~\ref{fig:subpools} shows the exact min--max envelope, median, and mean at
each pool size over all enumerated nonempty subpools.

\begin{figure*}[t]
\centering
\includegraphics[width=\linewidth]{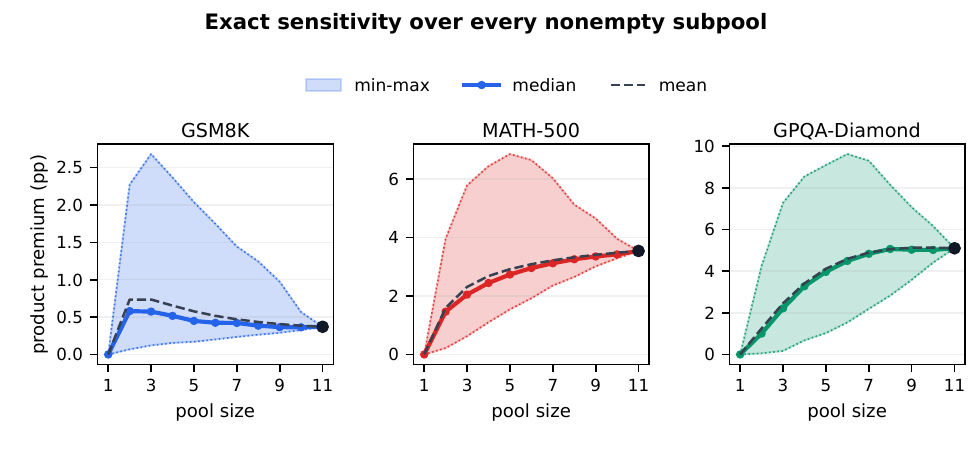}
\caption{Display-channel product premiums (pp) over every nonempty subpool,
grouped by size.  Shading is the exact minimum--maximum envelope of the
enumerated set, not sampling uncertainty; black points mark the full pool.}
\label{fig:subpools}
\end{figure*}

\subsection{What the released artifacts let a reader check}

The canonical Run-A results report PASS for all eleven declared scientific
invariants.  The submitted tensors, numerical contract, and machine-readable
outputs allow independent number checks, while the in-bundle manifest records
member hashes.  This evidence starts at frozen scores: it does not validate
the scorers or regenerate historical responses.  The added querywise summaries
and aggregation-order comparison use the same scores; their separate result
record, and the policy experiment's distinct chain, are identified in
Appendix~\ref{app:artifact-evidence}.

\section{Discussion}\label{sec:discussion}

\paragraph{What the audit establishes}
The results connect three questions that an oracle comparison needs to keep
separate.  Proposition~\ref{prop:bracket} gives the sharp set of any-correct
union probabilities compatible with the marginals: $[R_i,H_i]$.  When this
interval is non-degenerate, even exact marginal success probabilities leave
the union value unidentified.  Proposition~\ref{prop:decomposition} separates
the product-union gap into a policy-dependent distance from the single-commit
ceiling and a product-coupling premium.  The fixed-archive audit then shows how
the reported premium depends on the declared scorer, coupling, and pool, while
finite-$k$ paths describe sensitivity of its plug-in estimator.

\paragraph{Why this matters}
An oracle reference can be useful without being a deployable routing target.
If a distance to an any-correct reference is read as routing opportunity, the
interpretation needs a named coupling, a decision class and response budget,
a pool, and a scorer.  The coupling issue is a property of the functional
in Proposition~\ref{prop:bracket}, not a diagnosis of a particular benchmark's
implementation.  Across couplings compatible with the same marginals, zero premium is
admissible; this does not establish that the premium under a specified joint
sampling protocol is zero.  Our own product-union premium
is also a cross-decision-class contrast, not an exception to this reporting
requirement.  Under a specified independent joint protocol, product is the
corresponding union probability; its positive premium alone still does not
establish an attainable one-call gain.

\paragraph{Estimation bias is a separate question}
Under the iid within-cell and cross-model independence assumptions of
Appendix~\ref{app:conditional-bias}, the product plug-in is unbiased and the
ceiling plug-in has nonnegative bias.  The premium is therefore conservative
in expectation under those assumptions: its bias is non-positive.  This
explains how the optimizer's curse mentioned in the Introduction acts here;
it does not inflate the premium.  It also does not make a realized archived
premium a guaranteed lower bound, certify the archive's sampling assumptions,
or remove the need to specify the comparison's decision classes.

\paragraph{A recipe for evaluating a reader's router}
Appendix~\ref{app:archive-router} illustrates this calculation for a policy
evaluated on a held-out role within the historical archive, with the historical
exposure limitation stated separately.  The general recipe below specifies
the requirements for a reader's independently evaluated policy.
For a reader's router, Eq.~\eqref{eq:decomp} separates its distance from the
single-commit ceiling from the additional product-union premium.  First fix
the policy, model pool, scorer, and decoding protocol.  Evaluation queries and
outcomes must be excluded from training, tuning, and policy selection; scorer
development and model-pool selection must also be disclosed.  Freezing a policy
after inspecting those outcomes does not turn them into a holdout.  Record
$\pi_i$, its distribution over the declared models, using only information
available before the scored response.  A deterministic router gives a point
mass.  The conditional independence requirement in
Proposition~\ref{prop:decomposition} still applies.

Obtain repeated evaluation responses under the common protocol and compute
$\widehat p_{im}$ as in Eq.~\eqref{eq:plugin}.  A full-pool comparison requires
outcomes for every query--model cell, not just the models the router selected.
With the paper's equal query weights, estimate the policy's expected accuracy by
\[
\widehat{\bar q}=\frac{1}{N}\sum_i\sum_m\pi_i(m)\widehat p_{im}.
\]
Use those same marginals, queries, pool, scorer, and weights to compute
$\widehat{\bar R}$, $\widehat{\bar U}^{\perp}$, and
$\widehat{\bar\Delta}^{\perp}=\widehat{\bar U}^{\perp}-\widehat{\bar R}$.
Apply the maximum and product within each query before averaging
(Section~\ref{sec:estimation}).
The product union is the conditional union probability for independent
per-model resampling from these empirical marginals, not a certification of
the historical cross-model coupling.  These plug-in quantities are defined for
$k\ge1$ draws.  At $k=1$ the premium is forced to zero
(Proposition~\ref{prop:finitek}); it is also zero for a one-model pool.
$k>1$ and a pool with at least two models avoid those forced zeros but guarantee
neither a positive premium nor accurate estimation.

If $\widehat{\bar U}^{\perp}>\widehat{\bar q}$, form
\[
\widehat s=
\frac{\widehat{\bar\Delta}^{\perp}}
     {\widehat{\bar U}^{\perp}-\widehat{\bar q}}.
\]
Because $\pi_i$ is a probability distribution on that pool, the matched
calculation satisfies
$\widehat{\bar q}\le\widehat{\bar R}\le\widehat{\bar U}^{\perp}$ algebraically.
Thus $\widehat s\in[0,1]$ when its denominator is positive.  This is not a
finite-sample accuracy guarantee or a test for leakage: the inequality can
hold even for a policy selected using evaluation outcomes.  The holdout
requirement must be checked from the data and selection history.

For the exact quantities, Eq.~\eqref{eq:share} is defined only when
$\bar U^\perp>\bar q$.  Within the stated policy class, $0<s(q)<1$ means
both terms in Eq.~\eqref{eq:decomp} are positive, while $s(q)=1$ means
$\bar q=\bar R$ and a positive premium.  When $\bar\Delta^\perp>0$,
Corollary~\ref{cor:denominator} shows that the share rises as the comparator
approaches $\bar R$; it is not a policy-independent measure of routing
opportunity.  If $\bar\Delta^\perp=0$ and the denominator is positive, the
share is zero.  A zero denominator makes the share undefined; a negative one
cannot occur under these exact preconditions.  The term $\bar R-\bar q$ is
not proved learnable, and the union premium is not an equal-cost routing gain.

An exact $s(q)>1$ is incompatible with those preconditions.  Published
finite-sample estimates need more care: accuracy measured on separate realized
routed responses need not obey the matched-marginal inequality.  A share above
one can then reflect sampling disagreement or a mismatch in policy class,
protocol, split, scorer, or response reuse.  It is a reason to audit the
comparison, not proof from the numbers alone that a particular condition failed.
A value above one from the matched formula instead indicates an invalid input,
implementation, or reporting step.  An empirical report should distinguish the
policy's expected accuracy estimate from separately measured realized accuracy,
retain hats, and publish the component estimates and evaluation provenance.

\paragraph{Why the sensitivity axes matter together}
The scorer defines correctness; the coupling defines the joint event; the
pool changes both the union and the competing ceiling.  Finite $k$ instead
changes the empirical marginals supplied to nonlinear estimators, not the
population estimand.  These sensitivity sets have different constructions and
should not be ranked as if their ranges estimated a common population effect.
Reporting only one axis can leave the other choices hidden.

\paragraph{Reporting an interpretable comparison}
Table~\ref{tab:benchmark-contract} lists disclosures that help separate the
protocol, decision classes, target quantities, and empirical evidence.  It is
not a guarantee that a comparison is validated.  In particular, a normalized
router gap needs a named, independently evaluated policy; neither the hindsight
union nor the clairvoyant ceiling can supply its accuracy.  The historical
archive, outcome-informed scorer development, and other limits in
Section~\ref{sec:limitations} remain constraints on the present evidence.

\begin{table*}[t]
\centering
\caption{Recommended reporting contract for stochastic routing benchmarks.  The
rows separate inputs, estimands, sensitivity axes, and the extra evidence
needed for router-dependent claims.}
\label{tab:benchmark-contract}
\small
\setlength{\tabcolsep}{5pt}
\begin{tabular}{@{}L{0.24\linewidth}L{0.70\linewidth}@{}}
\toprule
Contract item & Minimum disclosure \\
\midrule
Protocol identity & Ordered queries and models; prompt; exact revisions;
decoder; response count; scorer implementation and fingerprint. \\
Decision class & Whether selection occurs before or after seeing responses,
and the model-call and verifier budget available to each comparator. \\
Target quantity & Absolute single-commit ceiling, named-coupling union, and
their premium before reporting a normalized gap. \\
Joint assumption & The cross-model coupling used for a union, together with
the marginally identified Fr\'echet interval. \\
Evaluation policy & Every material scorer channel and whether its policy was
fixed before output inspection or developed after it. \\
Sensitivity & Declared model pools and deterministic finite-draw paths,
including when and how each was fixed; their ranges must not be presented as
confidence intervals. \\
Router-dependent claim & An independently evaluated policy $q$, with its
training split, information set, and evaluation protocol. \\
Artifact lineage & Raw responses where releasable, identity manifests,
machine-readable results, hashes, and a deterministic number-to-table path. \\
\bottomrule
\end{tabular}
\end{table*}

\section{Limitations}\label{sec:limitations}

First, the responses are historical and exact model/tokenizer revisions were
not all captured; this is not a current leaderboard or a generation
reproduction.  Second, the fixed benchmarks and model pool were informed by
prior analysis of these archived tensors, so their selection was not
outcome-blind.  They are case studies, not random samples, and the analysis has
no confidence intervals, $p$-values, or cross-benchmark average.
Third, product is the joint law of specified independent generation and of
the independent-index empirical resampling experiment; neither certifies the
historical cross-model generation mechanism.  Same-index pairing is an
engineering convention.  Zero premium is attainable when only marginals are
fixed and couplings may vary, not a competing value within an already specified
independent joint protocol.

Fourth, the post-output-developed GSM8K and GPQA v2 scorers have only the
limited post-hoc human check in Appendix~\ref{app:human-audit}, not an
outcome-blind development holdout.  It leaves 4/30 GSM8K and 19/30 GPQA-Diamond
cells without definite two-person consensus.  Agreement on the determined
subsets does not establish archive-wide scorer accuracy, validate benchmark
references, or exclude shared human error.  The small, score-stratified sample
and the decision to forgo follow-up after inspecting the initial-judgment
summary further limit the check.  Reporting all retained scorer channels
reveals sensitivity but does not validate extraction semantics; the original
release reproduces frozen score-tensor-to-result computation, not the entire
raw-text-to-score path.
Fifth, $k=30$ is finite.  The prefix analysis does not prove stationarity,
exchangeability, or population convergence.  Sixth, exact enumeration removes
subpool-selection opacity only for this 11-model set; its extrema are
outcome-descriptive and do not generalize to unseen models.  The additional
policy illustration holds out evaluation scores from the current fitting and
tuning procedure but reuses this historically inspected archive; it does not
establish a pristine holdout or new-data generalization.  Finally, the
single-commit ceiling excludes cascades, parallel ensembles, response-aware
selection, and verifier-assisted sampling, and the product union uses more
model responses than a one-call policy.

\section{Conclusion}

Per-model success probabilities do not, in general, determine a hindsight
any-correct union under stochastic decoding.  Proposition~\ref{prop:bracket}
identifies the sharp interval, and Proposition~\ref{prop:decomposition}
separates a named product union from the ceiling for pre-response single-commit
policies.  These results distinguish an identification question from
finite-sample estimation and from the performance of a learned router.

In three fixed historical tensors, the named product coupling yields premiums
of 0.371, 3.542, and 5.100 percentage points on the display channels, with scorer
ranges of 0.359--0.398, 3.542--3.754, and 5.100--6.632 points.  For each
benchmark, the zero lower endpoint follows from the Fr\'echet identification
result over couplings with those marginals.  It does not negate a positive
premium under a specified independent joint protocol.  Exact subpool enumeration and frozen
finite-$k$ paths further show that neither full-pool maximality nor smooth
convergence should be assumed.
In these display-channel archives, GPQA-Diamond's larger mean product premium
than GSM8K's reflects a larger positive-premium fraction, not a larger
conditional mean among those queries.

The practical lesson is to report the conditional premium and its identification
interval, then make the scorer, pool, and finite-draw checks explicit.  For a
reader with an independently evaluated router, the decomposition also supplies
a way to state which part of its product-reference gap depends on that policy.
The archived policy illustration instantiates that comparison without
establishing a routing improvement or historical data independence.
The later human check supports scorer agreement only on its determined
subsets; unresolved cases, especially on GPQA-Diamond, and historical output
exposure remain limitations.

\bibliographystyle{IEEEtranN}
\bibliography{refs}

\begin{thebibliography}{26}
\providecommand{\natexlab}[1]{#1}
\providecommand{\url}[1]{#1}
\csname url@samestyle\endcsname
\providecommand{\newblock}{\relax}
\providecommand{\bibinfo}[2]{#2}
\providecommand{\BIBentrySTDinterwordspacing}{\spaceskip=0pt\relax}
\providecommand{\BIBentryALTinterwordstretchfactor}{4}
\providecommand{\BIBentryALTinterwordspacing}{\spaceskip=\fontdimen2\font plus
\BIBentryALTinterwordstretchfactor\fontdimen3\font minus
  \fontdimen4\font\relax}
\providecommand{\BIBforeignlanguage}[2]{{%
\expandafter\ifx\csname l@#1\endcsname\relax
\typeout{** WARNING: IEEEtranN.bst: No hyphenation pattern has been}%
\typeout{** loaded for the language `#1'. Using the pattern for}%
\typeout{** the default language instead.}%
\else
\language=\csname l@#1\endcsname
\fi
#2}}
\providecommand{\BIBdecl}{\relax}
\BIBdecl

\bibitem[Ong et~al.(2025)Ong, Almahairi, Wu, Chiang, Wu, Gonzalez, Kadous, and
  Stoica]{routellm}
I.~Ong, A.~Almahairi, V.~Wu, W.-L. Chiang, T.~Wu, J.~E. Gonzalez, M.~W. Kadous,
  and I.~Stoica, ``{RouteLLM}: Learning to route {LLMs} from preference data,''
  in \emph{International Conference on Learning Representations}, 2025.

\bibitem[Chen et~al.(2024{\natexlab{a}})Chen, Zaharia, and
  Zou]{chen2024frugalgpt}
L.~Chen, M.~Zaharia, and J.~Zou, ``{FrugalGPT}: How to use large language
  models while reducing cost and improving performance,'' \emph{Transactions on
  Machine Learning Research}, 2024.

\bibitem[Smith and Winkler(2006)]{smith2006optimizers}
J.~E. Smith and R.~L. Winkler, ``The optimizer's curse: Skepticism and
  postdecision surprise in decision analysis,'' \emph{Management Science},
  vol.~52, no.~3, pp. 311--322, 2006.

\bibitem[Ding et~al.(2024)Ding, Mallick, Wang, Sim, Mukherjee, R{\"u}hle,
  Lakshmanan, and Awadallah]{ding2024hybridllm}
D.~Ding, A.~Mallick, C.~Wang, R.~Sim, S.~Mukherjee, V.~R{\"u}hle, L.~V.~S.
  Lakshmanan, and A.~H. Awadallah, ``Hybrid {LLM}: Cost-efficient and
  quality-aware query routing,'' in \emph{International Conference on Learning
  Representations}, 2024.

\bibitem[Dekoninck et~al.(2025)Dekoninck, Baader, and
  Vechev]{dekoninck2025unified}
\BIBentryALTinterwordspacing
J.~Dekoninck, M.~Baader, and M.~Vechev, ``A unified approach to routing and
  cascading for {LLMs},'' in \emph{Proceedings of the 42nd International
  Conference on Machine Learning}, ser. Proceedings of Machine Learning
  Research, vol. 267, 2025, pp. 12\,987--13\,010. [Online]. Available:
  \url{https://proceedings.mlr.press/v267/dekoninck25a.html}
\BIBentrySTDinterwordspacing

\bibitem[Hu et~al.(2024)Hu, Bieker, Li, Jiang, Keigwin, Ranganath, Keutzer, and
  Upadhyay]{routerbench}
Q.~J. Hu, J.~Bieker, X.~Li, N.~Jiang, B.~Keigwin, G.~Ranganath, K.~Keutzer, and
  S.~K. Upadhyay, ``{RouterBench}: A benchmark for multi-{LLM} routing
  system,'' \emph{arXiv preprint arXiv:2403.12031}, 2024.

\bibitem[Huang et~al.(2025)Huang, Ling, Lin, Chen, Zhong, Wu, and
  Lin]{routereval}
\BIBentryALTinterwordspacing
Z.~Huang, G.~Ling, Y.~Lin, Y.~Chen, S.~Zhong, H.~Wu, and L.~Lin,
  ``{R}outer{E}val: A comprehensive benchmark for routing {LLM}s to explore
  model-level scaling up in {LLM}s,'' in \emph{Findings of the Association for
  Computational Linguistics: EMNLP 2025}, C.~Christodoulopoulos,
  T.~Chakraborty, C.~Rose, and V.~Peng, Eds.\hskip 1em plus 0.5em minus
  0.4em\relax Suzhou, China: Association for Computational Linguistics,
  November 2025, pp. 3860--3887. [Online]. Available:
  \url{https://aclanthology.org/2025.findings-emnlp.208/}
\BIBentrySTDinterwordspacing

\bibitem[Li et~al.(2026)Li, Zhang, Guo, Wang, Tang, Zhang, Chen, Qi, Ye, Bai,
  Wang, and Hu]{llmrouterbench}
H.~Li, Y.~Zhang, Z.~Guo, C.~Wang, S.~Tang, Q.~Zhang, Y.~Chen, B.~Qi, P.~Ye,
  L.~Bai, Z.~Wang, and S.~Hu, ``{LLMRouterBench}: A massive benchmark and
  unified framework for {LLM} routing,'' in \emph{Findings of the Association
  for Computational Linguistics: ACL 2026}, 2026, pp. 37\,733--37\,754.

\bibitem[Wang et~al.(2023)Wang, Wei, Schuurmans, Le, Chi, Narang, Chowdhery,
  and Zhou]{wang2023selfconsistency}
\BIBentryALTinterwordspacing
X.~Wang, J.~Wei, D.~Schuurmans, Q.~V. Le, E.~H. Chi, S.~Narang, A.~Chowdhery,
  and D.~Zhou, ``Self-consistency improves chain of thought reasoning in
  language models,'' in \emph{International Conference on Learning
  Representations}, 2023. [Online]. Available:
  \url{https://openreview.net/forum?id=1PL1NIMMrw}
\BIBentrySTDinterwordspacing

\bibitem[Chen et~al.(2024{\natexlab{b}})Chen, Davis, Hanin, Bailis, Stoica,
  Zaharia, and Zou]{chen2024morecalls}
\BIBentryALTinterwordspacing
L.~Chen, J.~Q. Davis, B.~Hanin, P.~Bailis, I.~Stoica, M.~Zaharia, and J.~Zou,
  ``Are more {LLM} calls all you need? towards the scaling properties of
  compound {AI} systems,'' in \emph{Advances in Neural Information Processing
  Systems}, vol.~37, 2024. [Online]. Available:
  \url{https://proceedings.neurips.cc/paper_files/paper/2024/hash/51173cf34c5faac9796a47dc2fdd3a71-Abstract-Conference.html}
\BIBentrySTDinterwordspacing

\bibitem[Wang et~al.(2025)Wang, Wang, Athiwaratkun, Zhang, and
  Zou]{wang2025moa}
\BIBentryALTinterwordspacing
J.~Wang, J.~Wang, B.~Athiwaratkun, C.~Zhang, and J.~Zou, ``Mixture-of-agents
  enhances large language model capabilities,'' in \emph{International
  Conference on Learning Representations}, 2025. [Online]. Available:
  \url{https://openreview.net/forum?id=h0ZfDIrj7T}
\BIBentrySTDinterwordspacing

\bibitem[Brown et~al.(2024)Brown, Juravsky, Ehrlich, Clark, Le, R\'{e}, and
  Mirhoseini]{brown2024monkeys}
\BIBentryALTinterwordspacing
B.~Brown, J.~Juravsky, R.~Ehrlich, R.~Clark, Q.~V. Le, C.~R\'{e}, and
  A.~Mirhoseini, ``Large language monkeys: Scaling inference compute with
  repeated sampling,'' \emph{arXiv preprint arXiv:2407.21787}, 2024. [Online].
  Available: \url{https://arxiv.org/abs/2407.21787}
\BIBentrySTDinterwordspacing

\bibitem[Song et~al.(2025)Song, Wang, Li, and Lin]{greedynondet}
Y.~Song, G.~Wang, S.~Li, and B.~Y. Lin, ``The good, the bad, and the greedy:
  Evaluation of {LLMs} should not ignore non-determinism,'' in
  \emph{Proceedings of the 2025 Conference of the Nations of the Americas
  Chapter of the Association for Computational Linguistics: Human Language
  Technologies (Volume 1: Long Papers)}, 2025, pp. 4195--4206.

\bibitem[Madaan et~al.(2024)Madaan, Singh, Schaeffer, Poulton, Koyejo,
  Stenetorp, Narang, and Hupkes]{madaan2024variance}
\BIBentryALTinterwordspacing
L.~Madaan, A.~K. Singh, R.~Schaeffer, A.~Poulton, S.~Koyejo, P.~Stenetorp,
  S.~Narang, and D.~Hupkes, ``Quantifying variance in evaluation benchmarks,''
  \emph{arXiv preprint arXiv:2406.10229}, 2024. [Online]. Available:
  \url{https://arxiv.org/abs/2406.10229}
\BIBentrySTDinterwordspacing

\bibitem[van Hasselt(2010)]{hasselt2010double}
\BIBentryALTinterwordspacing
H.~van Hasselt, ``Double {Q}-learning,'' in \emph{Advances in Neural
  Information Processing Systems}, vol.~23, 2010, pp. 2613--2621. [Online].
  Available:
  \url{https://proceedings.neurips.cc/paper/2010/hash/091d584fced301b442654dd8c23b3fc9-Abstract.html}
\BIBentrySTDinterwordspacing

\bibitem[Andrews et~al.(2024)Andrews, Kitagawa, and
  McCloskey]{andrews2024winners}
I.~Andrews, T.~Kitagawa, and A.~McCloskey, ``Inference on winners,'' \emph{The
  Quarterly Journal of Economics}, vol. 139, no.~1, pp. 305--358, 2024.

\bibitem[Berk et~al.(2013)Berk, Brown, Buja, Zhang, and
  Zhao]{berk2013postselection}
R.~Berk, L.~Brown, A.~Buja, K.~Zhang, and L.~Zhao, ``Valid post-selection
  inference,'' \emph{The Annals of Statistics}, vol.~41, no.~2, pp. 802--837,
  2013.

\bibitem[Fowler et~al.(2026)Fowler, Smith, Graviet, Myers, Greaves, Oozeer,
  Garc\'ia, Quirke, Abdullah, Barez, and Upadhyay]{fowler2026capability}
\BIBentryALTinterwordspacing
B.~Fowler, R.~Smith, D.~T. Graviet, W.~Myers, J.~Greaves, N.~F. Oozeer,
  A.~Garc\'ia, P.~Quirke, A.~Abdullah, F.~Barez, and S.~K. Upadhyay, ``The
  capability frontier: Benchmarks miss 82\% of model performance,'' 2026, oral
  paper at the 3rd DATA-FM Workshop at ICLR 2026; non-archival. [Online].
  Available: \url{https://arxiv.org/abs/2606.26836}
\BIBentrySTDinterwordspacing

\bibitem[Garg and Sagtani(2026)]{garg2026unsolvability}
\BIBentryALTinterwordspacing
S.~Garg and A.~Sagtani, ``Unsolvability ceiling in multi-{LLM} routing: An
  empirical study of evaluation artifacts,'' \emph{arXiv preprint
  arXiv:2605.07395}, 2026. [Online]. Available:
  \url{https://arxiv.org/abs/2605.07395}
\BIBentrySTDinterwordspacing

\bibitem[Chen(2026)]{chen2026cofailure}
\BIBentryALTinterwordspacing
J.~Chen, ``When does combining language models help? a co-failure ceiling on
  routing, voting, and mixture-of-agents across 67 frontier models,''
  \emph{arXiv preprint arXiv:2606.27288}, 2026. [Online]. Available:
  \url{https://arxiv.org/abs/2606.27288}
\BIBentrySTDinterwordspacing

\bibitem[Cobbe et~al.(2021)Cobbe, Kosaraju, Bavarian, Chen, Jun, Kaiser,
  Plappert, Tworek, Hilton, Nakano, Hesse, and Schulman]{gsm8k}
\BIBentryALTinterwordspacing
K.~Cobbe, V.~Kosaraju, M.~Bavarian, M.~Chen, H.~Jun, L.~Kaiser, M.~Plappert,
  J.~Tworek, J.~Hilton, R.~Nakano, C.~Hesse, and J.~Schulman, ``Training
  verifiers to solve math word problems,'' \emph{arXiv preprint
  arXiv:2110.14168}, 2021. [Online]. Available:
  \url{https://arxiv.org/abs/2110.14168}
\BIBentrySTDinterwordspacing

\bibitem[Hendrycks et~al.(2021)Hendrycks, Burns, Kadavath, Arora, Basart, Tang,
  Song, and Steinhardt]{mathdataset}
\BIBentryALTinterwordspacing
D.~Hendrycks, C.~Burns, S.~Kadavath, A.~Arora, S.~Basart, E.~Tang, D.~Song, and
  J.~Steinhardt, ``Measuring mathematical problem solving with the {MATH}
  dataset,'' in \emph{Proceedings of the Neural Information Processing Systems
  Track on Datasets and Benchmarks}, J.~Vanschoren and S.-K. Yeung, Eds.,
  vol.~1, 2021. [Online]. Available:
  \url{https://datasets-benchmarks-proceedings.neurips.cc/paper/2021/hash/be83ab3ecd0db773eb2dc1b0a17836a1-Abstract-round2.html}
\BIBentrySTDinterwordspacing

\bibitem[Lightman et~al.(2023)Lightman, Kosaraju, Burda, Edwards, Baker, Lee,
  Leike, Schulman, Sutskever, and Cobbe]{lightman2023verify}
\BIBentryALTinterwordspacing
H.~Lightman, V.~Kosaraju, Y.~Burda, H.~Edwards, B.~Baker, T.~Lee, J.~Leike,
  J.~Schulman, I.~Sutskever, and K.~Cobbe, ``Let's verify step by step,''
  \emph{arXiv preprint arXiv:2305.20050}, 2023. [Online]. Available:
  \url{https://arxiv.org/abs/2305.20050}
\BIBentrySTDinterwordspacing

\bibitem[{OpenAI}(2023)]{prm800k}
\BIBentryALTinterwordspacing
{OpenAI}, ``{PRM800K}: A process supervision dataset,'' GitHub repository,
  2023, documents and releases the fixed 500-problem held-out MATH split.
  [Online]. Available: \url{https://github.com/openai/prm800k}
\BIBentrySTDinterwordspacing

\bibitem[Rein et~al.(2024)Rein, Hou, Stickland, Petty, Pang, Dirani, Michael,
  and Bowman]{gpqa}
D.~Rein, B.~L. Hou, A.~C. Stickland, J.~Petty, R.~Y. Pang, J.~Dirani,
  J.~Michael, and S.~R. Bowman, ``{GPQA}: A graduate-level google-proof
  {Q}\&{A} benchmark,'' in \emph{First Conference on Language Modeling}, 2024.

\bibitem[Kydl\'{\i}\v{c}ek(2026)]{mathverify}
\BIBentryALTinterwordspacing
H.~Kydl\'{\i}\v{c}ek, ``{Math-Verify}: Math verification library,'' Python
  Package Index, version 0.9.0, 2026, policy-specific configuration is reported
  in the accompanying artifact. [Online]. Available:
  \url{https://pypi.org/project/math-verify/0.9.0/}
\BIBentrySTDinterwordspacing

\end{thebibliography}

\appendices

\section{Proofs and Counterexamples}\label{app:proofs}

\subsection{Proof of Proposition~\ref{prop:bracket}}

Let $E_{im}=\{B_{im}=1\}$.  Since
$E_{im}\subseteq\cup_{\ell}E_{i\ell}$ for every $m$,
\[
U_i^{\mu}=\Pr_{\mu_i}(\cup_m E_{im})
\ge\max_m\Pr(E_{im})=R_i.
\]
The union bound and the fact that a probability is at most one give
$U_i^\mu\le\min(1,\sum_m p_{im})$.

For attainability, let $Z$ be uniform on the unit circle.  Nested arcs
$E_{im}=[0,p_{im})$ attain the lower endpoint $R_i$.  For the upper endpoint,
place arcs of lengths $p_{im}$ consecutively around the circle; their union has
measure $\min\{1,\sum_m p_{im}\}$.  In either construction, the induced law of
$(\mathbf 1\{Z\in E_{i1}\},\ldots,\mathbf 1\{Z\in E_{iM}\})$ is a valid
coupling on $\{0,1\}^M$ with the required marginals.

Let $\mu_i^L$ and $\mu_i^H$ denote the lower- and upper-attaining couplings.
For $t\in[0,1]$, the mixture
$\mu_{i,t}=(1-t)\mu_i^L+t\mu_i^H$ retains those marginals.  Because
$U_i^\mu=1-\mu_i(0,\ldots,0)$ is affine in $\mu_i$,
\[
U_i^{\mu_{i,t}}
=(1-t)R_i+t\min\!\left\{1,\sum_m p_{im}\right\}.
\]
Thus every value in the interval is attainable.  Since $\mu_i$ is specified
separately for each query, endpoint-attaining couplings may be chosen querywise
before equal-weight averaging; hence the averaged endpoints are attainable.
Using the same mixture weight $t$ for every query sweeps every value between
those averaged endpoints, so the averaged identified set is the full interval.
The product coupling is admissible and therefore
$\Delta_i^\perp\ge0$.
\hfill$\square$

\subsection{Proof of Proposition~\ref{prop:decomposition}}

The law of total probability gives
\[
\begin{aligned}
\Pr(B_{iM_i}=1)
&=\sum_{m:\,\pi_i(m)>0}\Pr(M_i=m)\\
&\qquad\cdot\Pr(B_{im}=1\mid M_i=m),
\end{aligned}
\]
Conditional independence of the pre-response choice and the fresh scoring
response then gives $\Pr(B_{im}=1\mid M_i=m)=p_{im}$ on this support.  Restoring
the omitted zero-mass terms, the display equals
$\sum_m\pi_i(m)p_{im}=q_i$.
Because $\pi_i$ lies on the probability simplex,
\[
q_i=\sum_m\pi_i(m)p_{im}
\le\sum_m\pi_i(m)R_i=R_i.
\]
Adding and subtracting $\bar R$ from
$\bar U^\perp-\bar q$ proves Eq.~\eqref{eq:decomp}.  Non-negativity follows
from the pointwise policy ceiling and Proposition~\ref{prop:bracket}.
\hfill$\square$

\subsection{Proof of Corollary~\ref{cor:denominator}}

Let $d=\bar\Delta^\perp>0$.  Then
$s(q)=d/(\bar U^\perp-\bar q)$ and
$\bar U^\perp=\bar R+d$.  For $\bar q\le\bar R$ the denominator is positive
and strictly decreases as $\bar q$ increases, so $s(q)$ strictly increases.
\hfill$\square$

\subsection{Product-pool increments}\label{app:pool-increments}

For a pool $A$ and a new model $\ell\notin A$,
\[
U^\perp_{A\cup\{\ell\}}-U^\perp_A
=p_\ell(1-U^\perp_A)\ge0.
\]
Thus the product union is monotone with diminishing returns.  Since
$R_{A\cup\{\ell\}}-R_A=(p_\ell-R_A)_+$,
\[
\Delta^\perp_{A\cup\{\ell\}}-\Delta^\perp_A
=p_\ell(1-U^\perp_A)-(p_\ell-R_A)_+.
\]
The premium therefore need not be monotone.  For $A$ with marginals
$(0.5,0.5)$, $\Delta_A^\perp=0.25$; adding a model with marginal one makes
both union and ceiling one, and the premium becomes zero.

\subsection{Proof of Proposition~\ref{prop:finitek}}

At $k=1$, every empirical marginal is zero or one.  If all are zero, both
product union and maximum are zero; otherwise both are one.  Hence the premium
is zero.  At the full draw count, every permutation contains the same multiset
of indices, so all empirical marginals and derived deterministic functionals
coincide.

Nested prefixes do not imply monotonicity.  For model-1 and model-2 draw
sequences $b_{1,1:3}=(1,0,1)$ and $b_{2,1:3}=(0,1,1)$ in that order, the
premiums at $k=1,2,3$ are
$0$, $1/4$, and $2/9$, respectively.
\hfill$\square$

\subsection{Conditional population bias result}\label{app:conditional-bias}

This result is not assumed by the fixed-tensor headline.  Suppose nevertheless
that each cell contains $k$ iid Bernoulli responses with parameter $p_{im}$ and
that cell streams are independent across models for a fixed query.  Then
\[
\E[\widehat U_i^\perp]=U_i^\perp,\qquad
0\le\E[\widehat R_i]-R_i
\le\sqrt{\frac{\log M}{2k}},
\]
and
\[
-\sqrt{\frac{\log M}{2k}}
\le\E[\widehat\Delta_i^\perp]-\Delta_i^\perp\le0.
\]
The first identity follows by factorizing
$\E[\prod_m(1-\widehat p_{im})]$.  Jensen's inequality gives the lower
maximum-bias bound.  Hoeffding's lemma yields
$\E\exp\{\lambda(\widehat p_{im}-p_{im})\}
\le\exp\{\lambda^2/(8k)\}$; log-sum-exp and optimization at
$\lambda=\sqrt{8k\log M}$ give the upper bound for $M\ge2$.  When $M=1$,
$\widehat R_i=\widehat p_{i1}$ is unbiased and $U_i^\perp=R_i$, so both
relevant biases are zero.  Subtracting the nonnegative maximum bias from the
unbiased union establishes the premium statement.  At $M=11,k=30$ the bound is
$\sqrt{\log(11)/60}=0.1999$, far wider than the empirical premiums and not a
useful data-specific correction.  None of these population expectations is
inferred merely from the archived same-index layout.

\section{Additional Archived-Tensor Diagnostics}

\subsection{Same-index coupling diagnostic}\label{app:same-index}

For a draw prefix $J$, let a single index be sampled uniformly and reused
across models.  The resulting empirical coupling has marginals
$\widehat p_{im}(J)$, so its union lies inside the empirical Fr\'echet bracket.
Its difference from product may have either sign: for draw sequences
$b_{1,1:2}=(1,0)$ and $b_{2,1:2}=(1,0)$ the paired/product unions are
$0.5/0.75$; for $b_{1,1:2}=(1,0)$ and $b_{2,1:2}=(0,1)$ they are $1/0.75$.
A small difference therefore neither tests nor validates independence.

\subsection{Per-model descriptive accuracies and a best-fixed anchor}
\label{app:descriptive-anchor}

Table~\ref{tab:model-accuracy} reports per-model means for the display
channels.  For each benchmark and scorer separately, define
\[
\begin{aligned}
\widehat A&=\max_m\frac{1}{N}\sum_{i=1}^{N}\widehat p_{im},\\
\widehat G&=\widehat{\bar R}-\widehat A,\\
\widehat{\bar R}&=\frac{1}{N}\sum_{i=1}^{N}\max_m\widehat p_{im}.
\end{aligned}
\]
Here $\widehat A$ is the largest model column mean, selected and measured on
the same archive.  In contrast, $\widehat{\bar R}$ takes the maximum within
each query before averaging.  Every model's column mean is bounded by the
mean of the querywise maxima, so $\widehat G\ge0$.  The difference quantifies
the effect of these two aggregation orders in the fixed tensor.

This comparison was added in response to review, after the original results
were known.  Table~\ref{tab:best-fixed-anchor} uses all retained scorer
channels with the original full pool and full set of archived draws.  On the
display channels, the differences are 2.660 pp for GSM8K, 5.087 pp for
MATH-500, and 32.273 pp for GPQA-Diamond; the full-pool product premiums on
those same channels are 0.371, 3.542, and 5.100 pp, respectively
(Table~\ref{tab:core-results}).  The anchor is selected and measured
in-sample, so it is a selection-optimistic descriptive reference rather than
held-out policy performance.  It is not an independently evaluated
$\bar q$ in Eq.~\eqref{eq:decomp}; the difference $\widehat G$ does not establish
an attainable routing improvement.  The product premiums in
Tables~\ref{tab:core-results} and~\ref{tab:scorer-sensitivity} compare the
product union with the querywise ceiling instead of comparing aggregation orders.

\begin{table*}[ht]
\caption{In-sample aggregation-order comparison on the frozen scores, using
11 models and 30 archived draws per cell.  A dagger marks the existing
display channel.  The first two numerical columns are empirical score levels
(\%); their difference $\widehat G$ is in percentage points (pp).}
\label{tab:best-fixed-anchor}
\centering
\small
\begin{tabular}{@{}lrrr@{}}
\toprule
Scorer & Querywise ceiling & Best column mean & Difference \\
 & $\widehat{\bar R}$ (\%) & $\widehat A$ (\%) & $\widehat G$ (pp) \\
\midrule
\multicolumn{4}{@{}l}{GSM8K} \\
\texttt{gsm8k\_v2}$^{\dagger}$ & 98.980 & 96.320 & 2.660 \\
\texttt{gsm8k\_v1} & 98.887 & 95.880 & 3.007 \\
\texttt{gsm8k\_legacy} & 98.893 & 96.053 & 2.840 \\
\midrule
\multicolumn{4}{@{}l}{MATH-500} \\
\texttt{math500\_mathverify}$^{\dagger}$ & 87.713 & 82.627 & 5.087 \\
\texttt{math500\_legacy} & 83.687 & 77.267 & 6.420 \\
\midrule
\multicolumn{4}{@{}l}{GPQA-Diamond} \\
\texttt{gpqa\_v2}$^{\dagger}$ & 82.694 & 50.421 & 32.273 \\
\texttt{gpqa\_v1} & 84.680 & 45.673 & 39.007 \\
\texttt{gpqa\_legacy} & 87.778 & 50.572 & 37.205 \\
\bottomrule
\end{tabular}
\par\smallskip
\begin{minipage}{\linewidth}
\small
Differences are computed from unrounded exact values and then rounded;
subtracting the printed levels can differ in the last decimal place.
\end{minipage}
\end{table*}

\begin{table*}[ht]
\caption{Per-model empirical accuracy (\%) on display channels: mean binary
score over each benchmark's fixed queries and 30 archived draws per cell.}
\label{tab:model-accuracy}
\centering
\footnotesize
\begin{tabular}{@{}lrrr@{}}
\toprule
Model & GSM8K & MATH-500 & GPQA-Diamond \\
\midrule
Mistral-7B-Instruct-v0.3 & 53.460 & 14.927 & 13.199 \\
DeepSeek-R1-Distill-Qwen-7B & 85.907 & 53.687 & 15.354 \\
Qwen2.5-7B-Instruct-AWQ & 91.447 & 74.393 & 28.300 \\
Qwen2.5-14B-Instruct-AWQ & 93.640 & 77.740 & 35.859 \\
Qwen2.5-32B-Instruct-AWQ & 96.320 & 82.627 & 38.199 \\
Phi-4 & 94.647 & 80.560 & 50.421 \\
OLMo-2-1124-7B-Instruct & 79.973 & 36.240 & 26.734 \\
Yi-1.5-9B-Chat & 86.780 & 51.427 & 28.569 \\
Granite-3.3-8B-Instruct & 89.113 & 69.273 & 28.401 \\
Gemma-2-9B-IT & 87.387 & 48.980 & 22.071 \\
Llama-3.1-8B-Instruct & 85.613 & 49.440 & 23.232 \\
\bottomrule
\end{tabular}
\end{table*}

\subsection{Complete named-pool sensitivity}

Table~\ref{tab:leave-group-out} expands the leave-one-declared-group-out
portion of the fixed-pool audit.  Each row changes composition and pool size
together: the Qwen2.5 removal drops three models, whereas each other row drops
one.  The rows are descriptive joint contrasts and do not isolate composition,
pool size, or group-level causal effects.

\begin{table*}[t]
\centering
\caption{Display-channel product premiums (pp) after each declared group
removal from the fixed 11-model pool.  These contrasts change composition and
size together; they do not isolate family effects.}
\label{tab:leave-group-out}
\small
\begin{tabular}{@{}lrrr@{}}
\toprule
Removed group & GSM8K & MATH-500 & GPQA-Diamond \\
\midrule
Mistral & 0.326 & 3.460 & 4.482 \\
DeepSeek-Distill-Qwen & 0.327 & 3.500 & 4.885 \\
Qwen2.5 & 0.404 & 3.633 & 7.729 \\
Phi & 0.567 & 3.648 & 5.452 \\
OLMo & 0.366 & 3.423 & 5.022 \\
Yi & 0.364 & 3.403 & 5.158 \\
Granite & 0.352 & 3.311 & 4.408 \\
Gemma & 0.459 & 3.516 & 4.919 \\
Llama & 0.360 & 3.364 & 4.420 \\
\bottomrule
\end{tabular}
\end{table*}

\subsection{Complete finite-draw path summaries}

Table~\ref{tab:finite-k-full} reports every frozen $k$ used in the analysis.
The shared $k=30$ endpoints are exact full-archive identities, while the
intermediate min--max brackets summarize order sensitivity across the eight
precomputed permutations.

\begin{table*}[t]
\centering
\caption{Display-channel finite-$k$ summaries (pp): mean [minimum, maximum] across
eight fixed permutations.  Intermediate means diagnose prefix sensitivity;
within-$k$ ranges diagnose order sensitivity, not uncertainty coverage.
The $k=1$ zero and shared $k=30$ endpoint are structural, not convergence evidence.}
\label{tab:finite-k-full}
\small
\begin{tabular}{@{}rccc@{}}
\toprule
$k$ & GSM8K & MATH-500 & GPQA-Diamond \\
\midrule
1 & 0.000 [0.000, 0.000] & 0.000 [0.000, 0.000] & 0.000 [0.000, 0.000] \\
2 & 0.066 [0.000, 0.200] & 1.124 [0.925, 1.325] & 1.877 [1.326, 3.212] \\
5 & 0.236 [0.166, 0.349] & 2.372 [2.172, 2.502] & 3.559 [2.990, 4.201] \\
10 & 0.317 [0.238, 0.421] & 3.024 [2.873, 3.202] & 4.398 [3.796, 4.941] \\
15 & 0.337 [0.269, 0.380] & 3.251 [3.052, 3.396] & 4.662 [4.151, 5.078] \\
20 & 0.349 [0.312, 0.393] & 3.398 [3.208, 3.518] & 4.757 [4.505, 5.127] \\
25 & 0.358 [0.339, 0.380] & 3.478 [3.403, 3.533] & 4.909 [4.610, 5.185] \\
30 & 0.371 [0.371, 0.371] & 3.542 [3.542, 3.542] & 5.100 [5.100, 5.100] \\
\bottomrule
\end{tabular}
\end{table*}

\subsection{Submitted evidence and revision additions}\label{app:artifact-evidence}

The original submitted bundle contains the frozen tensors, numerical contract,
canonical Run-A outputs, analysis receipt, and per-member hash manifest.
It omits analysis source, the release-control files listed in
Section~\ref{sec:released-evidence},
raw responses, and the runtime.  The original frozen input manifest, numerical
contract, scientific results file, and submitted ZIP have SHA-256 prefixes
\texttt{a7c08b3a}, \texttt{1e334469}, \texttt{b95afa09}, and
\texttt{76e3fdf2}, respectively.  These identify the original submission,
not the additions below.

The post-review computations in Section~\ref{sec:querywise-premium} and
Appendix~\ref{app:descriptive-anchor} are recorded separately in
\texttt{descriptive\_results.json}, with SHA-256 prefix
\texttt{01f6fef1}.  The associated querywise records and analysis receipt
retain the frozen-input bindings.  These new outputs are not contained in the
original submitted ZIP and must accompany a revised evidence package.

Appendix~\ref{app:archive-router} uses a separate revision protocol and result
record, rather than relabeling the original no-policy-output release.  The
protocol, evaluation result, and independent arithmetic-verification record
have SHA-256 prefixes \texttt{710bf924}, \texttt{2e7c3b42}, and
\texttt{44561e0e}, respectively.  The result includes exact fractions for the
matched empirical decomposition.  The verification checks its arithmetic and
record bindings; it is not an independent refitting of the policy or evidence
of historical data independence.

\section{A historical-archive router illustration}\label{app:archive-router}

This post-review illustration supplies a policy-specific comparator for
Eq.~\eqref{eq:decomp}, using the existing MATH-500 archive rather than new
model responses.  It uses all 500 queries, the same 11-model pool, and 30
archived draws per query--model cell under the Math-Verify display channel.
The held-out designation applies to the fitting and selection procedure below,
not to the author's historical exposure to the archive.

\paragraph{Policy and split}
A fixed seeded split allocated 300 queries to TRAIN, 100 to TUNE, and 100 to
EVAL.  The split kept connected components of token-Jaccard similarity at
threshold $0.70$ within one role; this heuristic does not exclude every
semantic near-duplicate.  No query was excluded based on scores or premium.
The router uses TF-IDF query features and cosine nearest neighbors.
Vocabulary, inverse document frequencies, and model success rates were fitted
only on TRAIN.  TUNE selected the neighbor count from $\{5,10,25,50\}$,
yielding $K=25$, without refitting on TUNE.  Ties use fixed orders, with the
smaller $K$ preferred for a tuning tie.  The fixed-model comparator was also
selected only on TRAIN: Qwen2.5-32B-Instruct-AWQ.
The fitting process could read TRAIN and TUNE scores but not EVAL scores;
the prediction process received the fitted policy and EVAL query text, not
EVAL scores.  Its deterministic model choices were sealed before the separate
evaluation stage read those scores.  The run used CPU computation, with no
new response generation.

\paragraph{Matched empirical quantities}
For each EVAL query, let $a(i)$ be the sealed model choice and
$\widehat p_{im}=30^{-1}\sum_{j=1}^{30}B_{imj}$.
All quantities in Table~\ref{tab:archive-router} use these same marginals,
queries, scorer, pool, and equal query weights.  In particular,
\[
\widehat{\bar q}=\frac{1}{100}\sum_{i\in\mathrm{EVAL}}
\widehat p_{i,a(i)}.
\]
The ceiling $\widehat{\bar R}$ and product union
$\widehat{\bar U}^{\perp}$ use the within-query maximum and product of
Section~\ref{sec:estimation} before averaging over EVAL.
Here $\widehat{\bar q}$ is the mean archived success rate of the selected
models, equivalently the expected accuracy under uniform resampling of an
archived draw for each selection.  It is not accuracy measured on fresh
routed calls.  The fixed comparator uses its TRAIN-selected model for every
EVAL query.  The premium is
$\widehat{\bar\Delta}^{\perp}=\widehat{\bar U}^{\perp}-\widehat{\bar R}$,
and the share divides this premium by
$\widehat{\bar U}^{\perp}-\widehat{\bar q}$, which is positive for both
comparators.  The product union is the conditional union probability for
independent per-model archived-draw resampling, not the observed same-index
union or a certification of the historical generation coupling.

\begin{table}[t]
\caption{One matched EVAL split of the historical MATH-500 archive.
The fixed comparator is selected on TRAIN, unlike the in-sample anchor in
Appendix~\ref{app:descriptive-anchor}.  Levels are percentages, differences
are percentage points (pp), and shares are unitless.  Values are rounded
only for display.}
\label{tab:archive-router}
\centering
\small
\begin{tabular}{@{}lrr@{}}
\toprule
Quantity & TF-IDF/KNN router & TRAIN-selected fixed \\
\midrule
$\widehat{\bar q}$ (\%) & 79.6667 & 80.0333 \\
$\widehat{\bar R}$ (\%) & 85.0000 & 85.0000 \\
$\widehat{\bar U}^{\perp}$ (\%) & 88.2826 & 88.2826 \\
$\widehat{\bar\Delta}^{\perp}$ (pp) & 3.2826 & 3.2826 \\
$\widehat{\bar R}-\widehat{\bar q}$ (pp) & 5.3333 & 4.9667 \\
$\widehat{\bar U}^{\perp}-\widehat{\bar q}$ (pp) & 8.6160 & 8.2493 \\
$\widehat s$ & 0.3809943 & 0.3979288 \\
\bottomrule
\end{tabular}
\end{table}

The router did not outperform the TRAIN-selected fixed model on this split.
The example nevertheless makes the decomposition concrete: the same premium
has a different share when the comparison policy changes, as in
Corollary~\ref{cor:denominator}.  Neither share is a fraction of guaranteed
attainable routing improvement, and the ceiling gap is not shown to be
learnable.  This is a single descriptive illustration, not a new routing
method or a claim about performance on other queries.

\paragraph{Historical exposure and evidence scope}
The benchmark, pool, and outcomes had already been inspected during the
paper's development, and all 500 gold answers had been used in scorer
parsing and self-equivalence checks.  Data preparation and the host administrator
could access the complete selected score channel.  Restricting the fitting
and prediction stages therefore does not turn this archive into a previously
unseen holdout or establish historically independent research choices.
The matched inequalities
$\widehat{\bar q}\le\widehat{\bar R}\le\widehat{\bar U}^{\perp}$ and
$\widehat s\in[0,1]$ are algebraic checks, not tests for leakage.
No population confidence interval or deployment-performance inference is
claimed.  The separate post-review result record has SHA-256 prefix
\texttt{2e7c3b42}; it does not alter the original release's absence of named
policy outputs.

\section{Post-review Reference-Based Human Check}\label{app:human-audit}

\begin{table*}[t]
\caption{Human check of the existing frozen-v2 scores.  Match and mismatch
require two matching, definite initial human judgments; undetermined includes
every other pair.  Conditional match divides matches by the number with a
definite consensus, not by $N$.  All sampled cells are retained, and no
follow-up judgment was collected.  These are counts for a fixed stratified
sample, not estimates of archive-wide scorer accuracy.}
\label{tab:human-audit}
\centering
\small
\begin{tabular}{llrrrrl}
\toprule
Benchmark & Frozen-v2 side & $N$ & Match & Mismatch & Undetermined & Conditional match \\
\midrule
GSM8K & All & 30 & 26 & 0 & 4 & 26/26 \\
 & 0 & 15 & 12 & 0 & 3 & 12/12 \\
 & 1 & 15 & 14 & 0 & 1 & 14/14 \\
\midrule
GPQA-Diamond & All & 30 & 11 & 0 & 19 & 11/11 \\
 & 0 & 15 & 4 & 0 & 11 & 4/4 \\
 & 1 & 15 & 7 & 0 & 8 & 7/7 \\
\bottomrule
\end{tabular}
\end{table*}

\paragraph{Scope and sampling}
This check compares human judgments of archived final answers with the
existing GSM8K and GPQA-Diamond frozen-v2 scores.  It does not re-score the
archive or assess MATH-500.  Before sampling, the proposed workload was reduced
from 100 to 30 response cells per benchmark to keep unpaid participation
manageable.  The resulting 60 distinct cells contain 15 from each frozen-v2
score side per benchmark.  Selection excluded documented scorer-development
queries, used a fixed salted model rotation, and prohibited repeated queries
within each benchmark.  This is a score-stratified descriptive sample, not a
uniform sample of the archive.  Excluding known development cases cannot
establish that all prior exposure has been identified.

\paragraph{Task and blinding}
Three unpaid contributors each received 40 items: 20 from each benchmark.
Every cell received two initial judgments, and each contributor pair shared
20 cells.  The packets showed the question, its reference answer, and the raw
response, withholding automated decisions, model-identifying metadata, and other
contributors' judgments.  Contributors were instructed to work separately,
without AI or automated scoring tools, and to record a supported final-answer
span and a reason.  The common rules allowed \texttt{CORRECT},
\texttt{INCORRECT}, and \texttt{UNRESOLVED}; reference ambiguity, uncertain
record completeness, and expertise limits were reasons to remain unresolved.
The task was final-answer equivalence to the supplied reference, not a full
proof audit or independent verification of benchmark gold.  Procedural
blinding alone does not establish evaluator independence or expertise across
every GPQA subject; unchanged raw responses can also contain model clues.

\paragraph{Closure and a change to the planned procedure}
The original procedure allowed up to ten follow-up judgments per contributor.
After receiving the 120 initial judgments and inspecting their agreement
summary, but before comparing them with automated scores, the author decided
not to request further annotation, changing the planned procedure.
Matching \texttt{CORRECT}/\texttt{CORRECT}
or \texttt{INCORRECT}/\texttt{INCORRECT} pairs define an analysis consensus;
all other pairs remain undetermined, with both original labels retained.
The closure rule and resulting consensus records were locked before the
scorer comparison.  No follow-up, third-person adjudication, majority vote,
or author/AI replacement judgment was used.  All 60 cases remain reported;
conditional agreement is restricted to the definite-consensus subset.

\paragraph{Results}
Table~\ref{tab:human-audit} reports the comparison separately by benchmark
and original score stratum.  On GSM8K, the initial pairs comprise 26 matching
definite judgments, three with exactly one \texttt{UNRESOLVED}, and one with
both judgments \texttt{UNRESOLVED}.  On GPQA-Diamond, the corresponding counts
are 11, 15, and three, plus one \texttt{CORRECT}/\texttt{INCORRECT}
disagreement.  Opposite judgments and two unresolved judgments both remain
undetermined, not definite matches or mismatches with the scorer.

\paragraph{Interpretation and limits}
No scorer disagreement was observed among definite-consensus cases, but
GPQA-Diamond's low coverage prevents a broad validation claim.  Returned
reasons include subject-matter uncertainty, doubts about response completeness,
and question/reference ambiguity; the definite disagreement also reflects
different emphases on the final option and intermediate reasoning.  This
describes existing reasons, not new adjudications or grounds for dropping cases.

The determined subset cannot represent the unresolved cases or exclude
shared human mistakes.  We report these counts without a population estimate
or confidence interval.  This check provides partial evidence for these two
frozen-v2 channels, not validation of all models, channels, or historical
scorer-development choices.

\end{document}